\documentclass[12pt]{article}

\usepackage{array,booktabs}

\usepackage{color}

\usepackage{amsfonts,enumerate, comment}
\usepackage{apalike}

\usepackage{natbib}

\usepackage{enumerate}
\usepackage{amsmath}
\usepackage{amsthm}
\usepackage{mathtools}                  
\mathtoolsset{showonlyrefs}             
\usepackage{bbm}                        

\newtheorem*{theorem*}{Theorem}

\newtheorem{corollary}{Corollary}[section]

\newtheorem{proposition}{Proposition}
\newtheorem*{proposition*}{Proposition}
\newtheorem{lemma}{Lemma}
\newtheorem*{lemma*}{Lemma}

\usepackage{tikz}
\usetikzlibrary{shapes}
\usetikzlibrary{patterns}

\usepackage{subcaption, ragged2e,lscape}
\usepackage[colorlinks=true,linkcolor=blue,citecolor=blue]{hyperref}
\usepackage[top=1.5in, bottom=1.5in, left=1in, right=1in]{geometry}

\allowdisplaybreaks

\def\xb{\textnormal{\textbf{x}}}

\def\Xb{\textnormal{\textbf{X}}}

\def\Yb{\textnormal{\textbf{Y}}}
\def\Zb{\textnormal{\textbf{Z}}}
\def\1b{\textnormal{\textbf{1}}}

\newcommand{\eps}{\varepsilon}

\def\xb{{\mathbf{x}}}

\def\E{{\rm E}}

\def\supp{{\rm supp}}

\usepackage{amssymb}

\def\1#1{\mathds{1}_{#1}}
\usepackage{dsfont}

\definecolor{lightgrey}{RGB}{197,200,204}

\definecolor{green}{rgb}{0.55, 0.71, 0.0}

\tikzstyle{every picture}+=[remember picture]
\tikzstyle{every node}=[font=\scriptsize] 

\title{Consistency of the $k$-Nearest Neighbor Regressor under Complex Survey Designs}
\author{Caren Hasler}

\begin{document}

\maketitle

\begin{abstract}
We study the consistency of the $k$-nearest neighbor regressor under complex survey designs. While consistency results for this algorithm are well established for independent and identically distributed data, corresponding results for complex survey data are lacking. We show that the $k$-nearest neighbor regressor is consistent under regularity conditions on the sampling design and the distribution of the data. We derive lower bounds for the rate of convergence and show that these bounds exhibit the curse of dimensionality, as in the independent and identically distributed setting. Empirical studies based on simulated and real data illustrate our theoretical findings.
\end{abstract}


\section{Introduction}

The $k$-Nearest Neighbors ($k$-NN) algorithm is a non-parametric supervised learning method for classification and regression. 
As a classifier, $k$-NN assigns a unit to the most common class among its $k$ nearest neighbors. As a regressor, it predicts a unit's value by averaging the observed values of its $k$ nearest neighbors. To the best of our knowledge, the first formalization of $k$-NN is found in the seminal work of \cite{fix:hod:51}. Since then, interest in $k$-NN has grown rapidly in a wide range of fields, with survey sampling being no exception. 


Statistical consistency is a crucial asymptotic property. An estimator is consistent if it becomes more tightly concentrated around the parameter it intends to estimate as the sample size increases. The literature devoted to the consistency and other asymptotic properties of the $k$-NN classifier and regressor is vast. Indeed, since the seminal work of \cite{sto:77} on the consistency of nonparametric regression estimators, including $k$-NN, the subject has been the focus of numerous studies. To cite only a few of those, \cite{dev:gyo:lug:96,hal:par:sam:08,sam:12,gad:kle:mar:16,dor:gyo:wal:18,gyo:wei:21} present some asymptotic properties of the $k$-NN classifier and \cite{mac:81,dev:gyo:krz:lug:94,bia:cer:guy:10,bia:dev:duj:krz:12} study those of the $k$-NN regressor. The book of \cite{bia:dev:15:nn} is devoted to $k$-NN and discusses its asymptotic properties in depth.

All of the asymptotic results presented in the aforementioned studies are derived in a pure model-based framework, where the data are independent and identically distributed (i.i.d.). In this framework, a model for the data is assumed and the sample is viewed as a collection of draws of i.i.d. random variables. The population is treated as infinite, and the randomness comes from the data-generating process. In a design-based framework, however, the population is treated as finite and fixed. The sample data, or survey data, are selected at random from this population with a complex selection scheme, often without replacement and with unequal probabilities. The randomness then arises from the sampling process. Thus, the i.i.d. assumption does not hold for survey data. Given these differences, the asymptotic results about $k$-NN developed in a pure model-based framework are generally not directly valid in a design-based framework. To be valid for survey data, asymptotic results about $k$-NN need to be tailored to the design-based framework.

Despite the considerable interest in $k$-NN among survey researchers and practitioners, asymptotic results remain limited and, to the best of our knowledge, are available only in two contexts: imputation for nonresponse and spatial mapping in environmental studies.
In the context of imputation for nonresponse, missing values of nonrespondents are imputed using $k$-NN applied to the respondents. Five published papers examine asymptotic properties of $k$-NN imputation in surveys. \cite{che:sha:00} present asymptotic properties of the $1$-NN imputation method for functions of population means (or totals), population distributions, and population quantiles. \cite{che:sha:01} propose two asymptotically unbiased and consistent jackknife variance estimators for $1$-NN imputation. \cite{sha:wan:08} derive asymptotic properties of the point estimator, variance estimator, and confidence interval for population means and quantiles estimated via $1$-NN imputation. These three papers rely on the somewhat unrealistic assumption of a single covariate. 

Two other papers extend these results to the case of multiple covariates. In order to bypass the challenge of multiple covariates, \cite{yan:kim:2019:nni} assume that information contained in the covariates can be summarized by a single scalar variable and apply $1$-NN imputation to that variable. They study the asymptotic properties of the $1$-NN imputation estimator for general population parameters, including population means, proportions, and quantiles. Another solution to address multiple covariates is predictive mean matching. Predictive mean matching is implemented in two steps: 1) a model predicting the survey variable from the covariates is fitted using respondents' data, 2) $1$-NN imputation is applied to the predicted survey variable obtained from the fitted model. \cite{yan:kim:20} study this alternative and establish some asymptotic properties of the predictive mean matching estimator.

In the context of environmental studies, $k$-NN is applied as an interpolator. Specifically, it is used to map continuous populations with finite populations of areas within a design-based framework, allowing estimation of the value of the survey variable at unsampled areas. \cite{fat:mar:pis:pra:22} and \cite{fat:fra:pis:24} establish asymptotic properties of the $k$-NN interpolator in this context. 

The theoretical foundations of $k$-NN in the design-based framework are limited to the two aforementioned contexts. No theoretical foundations of the $k$-NN regressor or classifier have been derived in the design-based framework, despite the growing number of applications. In this paper, we aim to bridge the gap between theory and practice for $k$-NN in survey sampling. We establish the asymptotic properties of the $k$-NN regressor under complex survey designs. We adopt a superpopulation model approach. That is, we assume that the sample is obtained from a finite population with a possibly-complex survey design, as in the design-based approach. In addition, we assume that the population consists of realizations of i.i.d. random variables defined according to a model called the superpopulation model. The analysis is carried out according to both sources of randomness, namely the sampling design and the superpopulation model. In the case of the $k$-NN regressor, this choice seems natural. We elaborate on this choice in Section~\ref{section:notation}. In this paper, we list conditions on the sampling design and on the distribution of the data under which the $k$-NN regressor is consistent. We also provide lower bounds for the rate of convergence of the $k$-NN regressor. Moreover, we show that this estimator suffers from the curse of dimensionality, as in the traditional model-based approach. Finally, we present the results of empirical studies on simulated and real data, which illustrate and empirically support our theoretical findings.

This paper is organized as follows. Section~\ref{section:notation} contains a description of the context and pieces of notation. Section~\ref{section:knn} presents the $k$-NN regressor. The conditions required for consistency are listed and detailed in Section~\ref{section:conditions}. The main results on consistency and rate of convergence are stated in Section~\ref{section:consistency}. The proofs are in the same section, with some technical elements left in the \hyperref[section:propositions]{Appendix}. Section~\ref{section:empirical} presents the empirical studies carried out with simulated and true data. Concluding remarks are given in Section~\ref{section:conclusion}. 

\section{Context and Notation}\label{section:notation}

Let $U_v := \{1, 2, ...., N_v\}$ denote a finite population of size $N_v$. To each element $i$ of $U_v$, we associate a random vector $\left(\Xb_i, Y_i, \Zb_i\right) \in \mathbb{R}^p \times \mathbb{R} \times \mathbb{R}^r$. The covariates $\Xb_i$ are variables used to predict the survey variable $Y_i$. Variables $\Zb_i$ are design variables used to construct the sampling design, see next paragraph. The finite population data $D_{U_v} := {(\Xb_i, Y_i)}_{i \in U_v}$ are assumed to consist of i.i.d. realizations of a random vector $(\Xb, Y)$ generated from a superpopulation model $\xi$. For simplicity, we assume that the covariates and survey variable have bounded supports denoted by $\supp(\Xb)$ and $\supp(Y)$, respectively. Let $\E_\xi$ denote the expectation computed with respect to model $\xi$.

A random sample $S_v$ of size $n_v$ is drawn from $U_v$ using a sampling design $p_v(\cdot)$. The sample data, or survey data, $D_{S_v} := \{\left(\Xb_i, Y_i \right)\}_{i \in S_v}$ are not i.i.d. 
The sample membership indicators are $I_{vi} = \mathds{1}_{i \in S_v}$. For generic units $i$ and $j$ in $U_v$, the first- and second-order inclusion probabilities are $\pi_{vi}$ and $\pi_{vij}$, respectively. We build on the asymptotic framework of \cite{isa:ful:82} and consider that $U_v, v\in \mathbb{N}$ is a sequence of embedded finite populations and the population and sample sizes $N_v$ and $n_v$ grow to infinity as $v$ increases. The associated sequence of samples $S_v$ may not be embedded. Hereafter, we omit subscript $v$ to ease reading whenever it is not confusing to do so. That is, we may write $n$ for $n_v$, $N$ for $N_v$, $\pi_i$ for $\pi_{vi}$, and so on. We will keep in mind that these quantities depend on $v$. Let $\E_p$ denote the expectation computed with respect to sampling design $p$ conditionally on population data $D_{U_v}$. 

In this paper, we consider the problem of estimating the regression function $m(\xb):= \E_\xi \E_p \left[Y \rvert \Xb = \xb\right]$ for $\xb \in \supp(\Xb)$ from survey data $D_{S_v}$ where $\E_\xi \E_p$ is the expectation computed with respect to the distribution of the data and the sampling design. That is, we consider the problem of estimating the mean of $Y$ for given values of $\Xb$. An example is the mean wage of an employee ($Y$) paid by an establishment given the establishment location, its size, and the type of industry in which it is engaged ($\Xb$). Consider $\widehat{m}_n(\xb)$ an estimator of $m(\xb)$ obtained from survey data $D_{S_v}$. We are interested in the convergence of the estimator $\widehat{m}_n(\xb)$. That is, we would like to know if estimator $\widehat{m}_n(\xb)$ is getting closer to regression function $m(\xb)$ as the population and sample grow bigger. More precisely, we say that the sequence $\{\widehat{m}_n\}_{v \in \mathbb{N}}$ \emph{converges in $L^2$} towards regression function $m$ if
\begin{equation} \label{convergenceL2}
\lim_{v \to \infty} \E_\xi \E_p \left[\left\{\widehat{m}_n(\xb) - m(\xb)\right\}^2\right] = 0.
\end{equation}
We say in this case that estimator $\widehat{m}_n(\xb)$ is \emph{$L^2$-consistent}, hereafter shortened \emph{consistent}.

\section{$k$-Nearest Neighbors Estimators}\label{section:knn}

First consider the problem of estimating the superpopulation regression function $m(\xb)= \E_\xi \left[Y \rvert \Xb = \xb\right]$ for $\xb \in \supp(\Xb)$ from population data $D_{U_v}$. This case is unrealistic when using survey data, since some of the variables can only be observed at sample level. Nevertheless, this case is helpful for addressing the estimation from sample data. Let $B(\xb, \rho)$ be the closed ball of radius $\rho$ centered at $\xb$. Formally,
\begin{align}
    B(\xb, \rho) = \left\{\tilde{\xb} \in \mathbb{R}^d : \lVert \tilde{\xb} - \xb  \rVert \leq \rho   \right\},
\end{align}
where $\lVert \cdot  \rVert$ is the Euclidean norm. Consider $\Xb_{(1)}(\xb),\Xb_{(2)}(\xb),\ldots,\Xb_{(N)}(\xb)$ a reordering of $\Xb_1,\Xb_2,\ldots,\Xb_N$ in increasing values of $\lVert \Xb - \xb \rVert$. We denote by $\rho_{k_NU}(\xb)$ the maximum Euclidean distance between $\xb$ and the $k_N$ closest population units. That is, $\rho_{k_NU}(\xb) = \lVert \Xb_{(k_N)}(\xb) - \xb \rVert$. The closed ball of radius $\rho_{k_NU}(\xb)$ centered at $\xb$ is $B\left(\xb, \rho_{k_NU}(\xb)\right)$. The population units $i \in U$ such that $i \in B\left(\xb, \rho_{k_NU}(\xb)\right)$ are the $k_N$ population units closest to $\xb$ or the $k_N$ nearest neighbors of $\xb$ in the population.

Regression function $m(\xb)$ can be estimated using the finite population data $D_{U_v}$ with the finite population $k$-NN estimator
\begin{align}\label{estimator:mN}
  \widehat{m}_N(\xb) &= \frac{1}{k_N} \sum_{i \in U}  \1{i \in B\left(\xb, \rho_{k_NU}(\xb)\right)} Y_i \\
                    &= \left[ \sum_{j \in U}  \1{j \in B(\xb, \rho_{k_NU})} \right]^{-1} \sum_{i \in U}  \1{i \in B(\xb, \rho_{k_NU})} Y_i.
\end{align}
This estimator is the average $Y$-value of the $k_N$ population units closest to $\xb$. It is $L^q$-consistent for $m(\xb)$ for $q \geq 1$ when $k_N \to + \infty$, $k_N/N \to 0$, and $\E_\xi\left(|Y_i|^q\right) < +\infty$. That is,
\begin{align}
   \lim_{v \to \infty}	\E_\xi\left[ \lvert \widehat{m}_N(\xb) - m(\xb)\rvert^q\right] = 0
\end{align}
under these conditions. See Corollary 10.3 of \cite{bia:dev:15:nn}. 

%
%

For complex survey data $D_{S_v}$, the available data do not allow us to compute $\widehat{m}_N(\xb)$. In this work, we consider the problem of estimating the regression function $m(\xb)= \E_\xi \left[Y \rvert \Xb = \xb\right]$ using the $k$-NN algorithm applied on survey data $D_{S_v}$. Consider $\Xb_{(1S)}(\xb)$, $\Xb_{(2S)}(\xb)$,$\ldots$, $\Xb_{(nS)}(\xb)$ a reordering of $\left\{\Xb_i\right\}_{i \in S}$ in increasing value of $\lVert \Xb - \xb \rVert$. We note $B\left(\xb, \rho_{k_nS}(\xb)\right)$ the closed ball of radius $\rho_{k_nS}(\xb)$ centered at $\xb$ where $\rho_{k_nS}(\xb)$ is the maximum Euclidean distance between $\xb$ and the $k_n$ closest sample units. That is, $\rho_{k_nS}(\xb) = \lVert \Xb_{(k_nS)}(\xb) - \xb \rVert$. The sample units $i \in S$ such that $i \in B\left(\xb, \rho_{k_nS}(\xb)\right)$ are the $k_n$ sample units closest to $\xb$ or the $k_n$ nearest neighbors of $\xb$ in the sample. The left panel of Figure~\ref{fig:neighborhoods:partition} illustrates the difference between $B\left(\xb, \rho_{k_nS}(\xb)\right)$ and $B\left(\xb, \rho_{k_NU}(\xb)\right)$ for $k_n = k_N = 4$.

The sample $k_n$-NN estimator is
\begin{align}\label{estimator:mn}
  \widehat{m}_n(\xb) = \left[ \sum_{j \in U}  \1{j \in B\left(\xb, \rho_{k_nS}(\xb)\right)}\frac{I_j}{\pi_j} \right]^{-1} \sum_{i \in U}  \1{i \in B\left(\xb, \rho_{k_nS}(\xb)\right)}\frac{I_i}{\pi_i} Y_i.
\end{align}
This estimator is a survey weighted average $Y$-value of the $k_n$ sample units closest to $\xb$. The right panel of Figure~\ref{fig:neighborhoods:partition} illustrates how this estimator generates a partition of the support of $\Xb$ into polygons. All $\xb$ within a polygon have the same nearest neighbors, and therefore the same estimated regression function $\widehat{m}_n(\xb)$. We denote by $Q_n$ this partition of the support of $\Xb$ into polygons. 

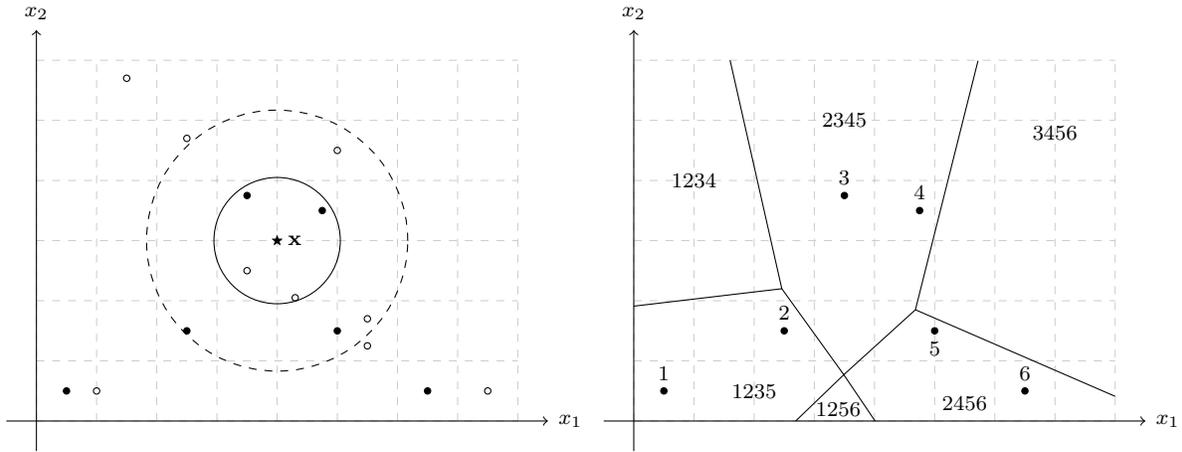
\begin{figure}
\begin{center}
           \begin{tikzpicture}[scale = 0.8]
            every label/.append style = {font=\tiny}
            \tikzstyle{every node}=[font=\scriptsize]
                \draw[help lines, color=gray!40, dashed] (-4,-3) grid (4,3);
                \draw[->] (-4.5,-3)--(4.5,-3) node[right]{$x_1$};
                \draw[->] (-4,-3.5)--(-4,3.5) node[above]{$x_2$};


                \draw (0,0) ellipse (2.17cm and 2.17cm)[color=black, dashed];
                \draw (0,0) ellipse (1.05cm and 1.05cm)[color=black];

                \node [star, star point height=1pt, minimum size=0.1cm, draw, fill = black, inner sep = 0.5pt]
       at (0,0) {};
                \node [draw = none, anchor = west]
       at (0,0) {$\xb$};

                \draw[black]   (0.3,-0.95) circle (1.5pt) node{};
                \draw[black]   (1,1.5) circle (1.5pt) node{};
                \draw[black, fill = black]   (0.75, 0.5) circle (1.5pt) node{};
                \draw[black, fill = black]  (-0.5,0.75) circle (1.5pt) node{};
                \draw[black]  (-0.5,-0.5) circle (1.5pt) node{};

                \draw[black]  (-1.5,1.7) circle (1.5pt) node{};
                \draw[black]  (1.5,-1.3) circle (1.5pt) node{};
                \draw[black, fill = black]  (-1.5,-1.5) circle (1.5pt) node{};
                \draw[black, fill = black]  (1,-1.5) circle (1.5pt) node{};
                \draw[black]  (1.5,-1.75) circle (1.5pt) node{};

                \draw[black]  (-2.5,2.7) circle (1.5pt) node{};
                \draw[black, fill = black]  (2.5,-2.5) circle (1.5pt) node{};
                \draw[black, fill = black]  (-3.5,-2.5) circle (1.5pt) node{};
                \draw[black]  (-3,-2.5) circle (1.5pt) node{};
                \draw[black]  (3.5,-2.5) circle (1.5pt) node{};
            \end{tikzpicture}
            \begin{tikzpicture}[scale = 0.8]
            every label/.append style = {font=\tiny}
            \tikzstyle{every node}=[font=\scriptsize]
                \draw[help lines, color=gray!40, dashed] (-4,-3) grid (4,3);
                \draw[->] (-4.5,-3)--(4.5,-3) node[right]{$x_1$};
                \draw[->] (-4,-3.5)--(-4,3.5) node[above]{$x_2$};

                \draw[black, fill = black]   (0.75, 0.5) circle (1.5pt) node{};
                \draw[black, fill = black]  (-0.5,0.75) circle (1.5pt) node{};
                \draw[black, fill = black]  (-1.5,-1.5) circle (1.5pt) node{};
                \draw[black, fill = black]  (1,-1.5) circle (1.5pt) node{};
                \draw[black, fill = black]  (2.5,-2.5) circle (1.5pt) node{};
                \draw[black, fill = black]  (-3.5,-2.5) circle (1.5pt) node{};

                \node [draw = none, anchor = south] at (-3.5,-2.5) {1};
                \node [draw = none, anchor = south] at (-1.5,-1.5) {2};
                \node [draw = none, anchor = south] at (-0.5,0.75) {3};
                \node [draw = none, anchor = south] at (0.75, 0.5) {4};
                \node [draw = none, anchor = north] at (1,-1.5)    {5};
                \node [draw = none, anchor = south] at (2.5,-2.5)  {6};

                \draw (-2.40, 3) -- (-1.54,-0.80) -- (-4,-1.09);
                \draw (-1.31,-3)  -- (-0.51,-2.23)--(0.01,-3);
                \draw (0.68,-1.15) -- (1.72,2.99);
                \draw (-1.54,-0.80) -- (-0.51,-2.23)--(0.68,-1.15);
                \draw (0.68,-1.15) -- (4,-2.59);

                \node [draw = none, anchor = center] at (-3,1) {1234};
                \node [draw = none, anchor = center] at (-0.5,2) {2345};
                \node [draw = none, anchor = center] at (-2,-2.5) {1235};
                \node [draw = none, anchor = south] at (-0.6,-3.1) {1256};
                \node [draw = none, anchor = south] at (1.5,-3) {2456};
                \node [draw = none, anchor = south] at (3,1.5) {3456};
%
%
%
%
            \end{tikzpicture}
\end{center}
  \caption{Left panel: 15 population units (filled and empty dots), six sample units (filled dots), closed ball of radius $\rho_{4U}(\xb)$ centered at $\xb$ ($B\left(\xb, \rho_{4U}(\xb)\right)$, solid circle), and closed ball of radius $\rho_{4S}(\xb)$ centered at $\xb$ ($B\left(\xb, \rho_{4S}(\xb)\right)$, dashed circle).\\
  Right panel: six sample units (filled dots), partition of $\mathbb{R}^2$ obtained from the 4-nearest neighbors applied to sample units (polygons), labels of four nearest sample units of any point in a polygon (four-figure numbers within the polygons).}\label{fig:neighborhoods:partition}
\end{figure}

The final goal is to show that estimator $\widehat{m}_n(\xb)$ is $L^2$-consistent for regression function $m(\xb)$ and to provide a lower bound for the rate of convergence. We follow the idea of \cite{tot:elt:11:tree} and introduce a hypothetical estimator
\begin{align}\label{estimator:mstar}
  \widehat{m}^*_n(\xb) = \left[ \sum_{j \in U}  \1{j \in B\left(\xb, \rho_{k_nS}(\xb)\right)} \right]^{-1} \sum_{i \in U}  \1{i \in B\left(\xb, \rho_{k_nS}(\xb)\right)} Y_i,
\end{align}
where a population unit $i \in U$ is in $B\left(\xb, \rho_{k_nS}(\xb)\right)$ if it is no further to $\xb$ than the $k_n$ closest sample units. Hence, the indicator in this sum is 1 if and only if a unit is in the set
\begin{align}
  \left\{i \in U ; i \in B\left(\xb, \rho_{k_nS}(\xb)\right)\right\} = \left\{i \in U ; \lVert \Xb_i - \xb \rVert \leq  \lVert \Xb_{(k_nS)}(\xb) - \xb \rVert \right\}.
\end{align}
Note that the size of this set is random and larger than $k_n$. Estimator $\widehat{m}^*_n(\xb)$ computes the average of $Y$-values for the population units $i$ in the ball of radius $\rho_{k_nS}(\xb)$ centered at $\xb$. This estimator cannot be computed based on the sample data $D_{S_v}$. It is however useful for the proof of consistency.

Let us illustrate the different estimators using the left panel of Figure~\ref{fig:neighborhoods:partition}. For four nearest neighbors, the population estimator $\widehat{m}_N(\xb)$ is based on the four population units closest to $\xb$, shown as the units inside the solid circle. The sample estimator $\widehat{m}_n(\xb)$ uses the four sampled units closest to $\xb$, corresponding to the filled dots inside the dashed circle. Finally, the hypothetical estimator $\widehat{m}^*_n(\xb)$ is computed from all population units that lie no farther from $\xb$ than the fourth closest sampled unit, corresponding to all dots, empty and filled, inside the dashed circle.

The general idea of the proof is as follows. In a first step, we show that the hypothetical estimator $\widehat{m}^*_n(\xb)$ is $L^2$-consistent for $m(\xb)$. We also give a lower bound for the rate of convergence. In a second step, we show that the sample estimator $\widehat{m}_n(\xb)$ is $L^2$-design-consistent for the hypothetical estimator $\widehat{m}^*_n(\xb)$. We also give a lower bound for the rate of convergence. In a third and last step, we use the output of the first two steps to show that the sample estimator $\widehat{m}_n(\xb)$ is $L^2$-consistent for $m(\xb)$ and provide a lower bound for the rate of convergence.

\section{Conditions}\label{section:conditions}

We impose several conditions on the population, the regression function, the sampling design, and the sequence $k_n$ for the proof of consistency of the sample estimator $\widehat{m}_n(\xb)$. These conditions are listed below, followed by further details and explanations.

\begin{enumerate}[({C}1):]
    \item\label{condition:y:bounded} There exists a constant $M_1>0$ such that $\max\limits_{i \in U}|y_i|\leq M_1 < +\infty$ with $\xi$-probability one.
    \item\label{condition:lipschitz} There exists a constant $L>0$ such that $|m(\xb)-m(\tilde{\xb})|\leq L \lVert \xb-\tilde{\xb} \rVert $ for all $\xb,\tilde{\xb}$ in $\supp(\Xb)$.
    \item\label{condition:variance:residuals} There exists a constant $\sigma^2>0$ such that $\sup\limits_{\xb \in \supp(\Xb)}\E_{\xi}\left[\left. \left\{\Yb - m(\Xb)\right\}^2 \right| \Xb = \xb\right] \leq \sigma^2$.
    \item\label{condition:density:neighborhoods} There exists a constant $C > 0$ such that
    $$ \sum_{i \in U } \1{i \in B\left(\xb, \rho_{k_nS}(\xb)\right)} \leq C\frac{N}{n}k_n$$ for all $\xb$ in $\supp(\Xb)$ and all $v$ with $\xi$-probability one.
    \item\label{condition:sampling:fraction} The sampling fraction satisfies $\lim\limits_{v \to +\infty} nN^{-1} = f \in (0,1)$ with $\xi$-probability one.
    \item\label{condition:non:informative} The sampling design is non-informative.
    \item\label{condition:pi} There exists a constant $\lambda$ such that $\pi_{i}\geq \lambda > 0$ for all $v$.
    \item\label{condition:pi2} The second order inclusion probabilities satisfy $\max\limits_{i,j\in U,i \neq j} \left|\frac{\pi_{ij}}{\pi_i \pi_j}-1 \right| = O(n^{-1})$ with $\xi$-probability one.
    \item\label{condition:I2}  There exists a constant $M_r$ such that the conditional expectation of the sample membership indicators satisfies $\left|\E_p(I_iI_j|Q_n) - \pi_{ij}\right| \leq M_r n^{-1}$ for all $i,j\in U$ with $p$-probability one and $\xi$-probability one.
    \item\label{condition:kn} $k_n \to + \infty$.
    \item\label{condition:kn:n} $k_n/n \to 0$.
\end{enumerate}

Condition (C\ref{condition:y:bounded}) states that the survey variable is bounded. Condition (C\ref{condition:lipschitz}) states that the regression function is Lipschitz continuous. Condition (C\ref{condition:variance:residuals}) states that the model residuals have bounded conditional variance.

Condition (C\ref{condition:density:neighborhoods}) is about the density of observations in the neighborhoods. It can be rewritten
\begin{align}\label{equation:density:neighborhoods}
    \frac{n}{N} \slash \frac{\sum_{i \in S } \1{i \in B\left(\xb, \rho_{k_nS}(\xb)\right)}}{\sum_{i \in U } \1{i \in B\left(\xb, \rho_{k_nS}(\xb)\right)}}  \leq C
\end{align}
since $k_n = \sum_{i \in S } \1{i \in B\left(\xb, \rho_{k_nS}(\xb)\right)}$. This means that the sampling fraction in the ball of radius $\rho_{k_nS}(\xb)$ centered at $\xb$ is of the same order as the general sampling fraction. This avoids having some neighborhoods becoming overly under- or over-represented in the sample as the population and sample grow. We present cases in which this condition is satisfied or violated in Section~\ref{section:sim:data:condition:density}.

Condition (C\ref{condition:sampling:fraction}) ensures that the sampling fraction does not converge towards 0 or 1. For instance, this condition is satisfied if the sequence of sampling fractions is constant. However, this condition is not satisfied if the sample grows systematically faster than the population, or inversely. For instance, suppose that $n = \sqrt{N}$ (rounded to the nearest integer). In this case, the population grows systematically faster than the sample and the sampling fraction $nN^{-1} = N^{-1/2}$ converges to 0. Inversely, suppose that $n = \exp(-N^{-1})N$ (rounded to the nearest integer). In this case, we have $ 0 < n < N$ and the sampling fraction $nN^{-1} = \exp(-N^{-1})$ converges to 1.

Condition (C\ref{condition:non:informative}) means that the sample indicators $I_i$ and the survey variables $Y_i$ are conditionally independent given the values of the design variables $\{\Zb_i\}_{i \in U}$.

Condition (C\ref{condition:pi}) states that the first order inclusion probabilities are bounded below. This is for instance the case for simple random sampling without replacement (srswor) and cluster sampling when the sequence of sampling fraction is bounded below. For stratified sampling, this condition is satisfied when the number of units selected in each stratum grows at the same rate as the number of population units in the stratum. However, this condition may not be satisfied for stratified sampling design if the number of population units in a stratum grows much faster than the number of selected units in this stratum. For instance, consider a stratum $h$ of size $N_h$ in which we select $n_h$ units with simple random sampling without replacement. Suppose that we have $n_h = O(\sqrt{N_h})$. For such a stratum, there exists positive numbers $M$ and $v_0$ such that $n_h \leq M\sqrt{N_h}$ for all $v\geq v_0$. This means that we have $\pi_i \leq M/\sqrt{N_h}$ for all $v\geq v_0$ and all units $i$ in stratum $h$. If the size of the stratum $N_h$ grows to infinity, the inclusion probabilities $\pi_i$ converge to 0 for all units in that stratum and the inclusion probabilities are not bounded below.

Condition (C\ref{condition:pi2}) is about the dependence of the sample membership indicators. It is satisfied for any sampling design whose sample membership indicators are not too dependent. Indeed, quantity $\left|\frac{\pi_{ij}}{\pi_i \pi_j}-1 \right|$ can be viewed as a measure of dependence between the sample membership indicators of units $i$ and $j$. The stronger the dependence, the bigger this measure. This measure is null for a design with independent sample membership indicators. For instance, this measure is null for Poisson sampling design since $\pi_{ij} = \pi_i \pi_j$. As a result, Condition (C\ref{condition:pi2}) is satisfied for this sampling design. This condition is also satisfied for srswor. For this sampling design, we have $\pi_{ij} = \frac{n(n-1)}{N(N-1)} < \pi_i\pi_j = \frac{n^2}{N^2}$ and
\begin{align}
  \left|\frac{\pi_{ij}}{\pi_i \pi_j}-1 \right| &= 1 - \frac{\pi_{ij}}{\pi_i \pi_j}
                                                = \frac{1}{\pi_i \pi_j}\left(\pi_i \pi_j - \pi_{ij}\right)
                                                = \frac{N^2}{n^2}\left( \frac{n^2}{N^2} - \frac{n(n-1)}{N(N-1)} \right)\\
                                                &= \frac{N}{n}\left( \frac{n}{N} - \frac{n-1}{N-1} \right)
                                                = \frac{N-n}{n(N-1)}
\end{align}
for all $i,j \in U$. This quantity is $O(n^{-1})$.

Condition (C\ref{condition:pi2}) is also satisfied for any stratified sampling designs that are not too highly stratified. Indeed, suppose that units are selected within strata using srswor. The sample size selected in a generic stratum $U_h$ is $n_h$ and the stratum size is $N_h$. For two distinct units $i,j \in U_h, i \neq j$ the first order inclusion probability is $\pi_i = n_h/N_h$ and the second order probabilities
\begin{align}
  \pi_{ij} = \left\{
  \begin{array}{ll}
    \frac{n_h(n_h-1)}{N_h(N_h-1)} & \hbox{if $i,j\in U_h$;} \\
    \frac{n_hn_\ell}{N_hN_\ell} & \hbox{if $i\in U_h,j\in U_\ell, h \neq \ell$.}
  \end{array}
\right.
\end{align}
Similar computations to those for srswor yield
\begin{align}
 \left|\frac{\pi_{ij}}{\pi_i \pi_j}-1 \right| = \left\{
  \begin{array}{ll}
    \frac{N_h-n_h}{n_h(N_h-1)} & \hbox{if $i,j\in U_h$;} \\
    0 & \hbox{if $i\in U_h,j\in U_\ell, h \neq \ell$.}
  \end{array}
\right.
\end{align}
We obtain $\max\limits_{i,j\in U,i \neq j} \left|\frac{\pi_{ij}}{\pi_i \pi_j}-1 \right| = \max\limits_{h} \frac{N_h-n_h}{n_h(N_h-1)}$. When the minimum number $n_h$ of units selected in all stratum increases at the same rate as $n$ does, quantity $\max\limits_{h} \frac{N_h-n_h}{n_h(N_h-1)}$ is $O(n^{-1})$ and Condition (C\ref{condition:pi2}) is satisfied. However, this condition is not stratified if the sequence of stratified sampling design is so highly stratified that the minimum number of units selected in a stratum is not increasing or is increasing much slower than the sample size $n$.

Condition (C\ref{condition:pi2}) also fails to hold for sampling designs whose sample indicators are highly dependent. For instance, this condition is not satisfied for systematic sampling. Indeed, consider $K=N/n$ with $K$ integer. Then $\pi_i = n/N$ and
\begin{align}
  \pi_{ij} = \left\{
  \begin{array}{ll}
    \frac{n}{N} & \hbox{if $i\mod K = j\mod K$} \\
    0 & \hbox{otherwise.}
  \end{array}
\right.
\end{align}
We obtain
\begin{align}
 \left|\frac{\pi_{ij}}{\pi_i \pi_j}-1 \right| = \left\{
  \begin{array}{ll}
    \left|\frac{N}{n}-1\right| & \hbox{if $i\mod K = j\mod K$} \\
    1 & \hbox{otherwise.}
  \end{array}
\right.
\end{align}
and $\max\limits_{i,j\in U,i \neq j} \left|\frac{\pi_{ij}}{\pi_i \pi_j}-1 \right| \geq 1$. This quantity is not $O(n^{-1})$. Condition (C\ref{condition:pi2}) also fails to hold for cluster sampling. Consider that $t$ clusters are selected with equal probabilities from $T$ clusters. Then $\pi_i = t/T$ and
\begin{align}
  \pi_{ij} = \left\{
  \begin{array}{ll}
    \frac{t}{T} & \hbox{if $i$ and $j$ belong to the same cluster,} \\
    \frac{t}{T}\frac{t-1}{T-1} & \hbox{if $i$ and $j$ belong to different clusters.}
  \end{array}
\right.
\end{align}
We obtain
\begin{align}
 \left|\frac{\pi_{ij}}{\pi_i \pi_j}-1 \right| = \left\{
  \begin{array}{ll}
    \frac{T}{t}-1 & \hbox{if $i$ and $j$ belong to the same cluster,} \\
    1-\frac{t-1}{T-1}\frac{T}{t} & \hbox{if $i$ and $j$ belong to different clusters.}
  \end{array}
\right.
\end{align}
and $\max\limits_{i,j\in U,i \neq j} \left|\frac{\pi_{ij}}{\pi_i \pi_j}-1 \right| \geq \frac{T}{t}-1$. This quantity is not $O(n^{-1})$.

Condition (C\ref{condition:I2}) states that knowing partition $Q_n$ provides less and less information about the second order inclusion probabilities as the population and sample increase. Consider the random variable
\begin{align}\label{rij}
  r_{ij} &= \E_p(I_iI_j|Q_n) - \pi_{ij}.
\end{align}
This variable represents the difference in the second order probability when partition $Q_n$ is known versus unknown. Condition~(C\ref{condition:I2}) implies that
\begin{align}
    \E_p\left(\max\limits_{i,j\in U}|r_{ij}|\right)&=O(n^{-1}) \quad \mbox{with $\xi$-probability one,} \label{equation:Ep:rij}\\
    \E_\xi\E_p\left(\max\limits_{i,j\in U}|r_{ij}|\right)&=O(n^{-1}). \label{equation:Exip:rij}
\end{align}
This condition is complicated to study because it is often impossible to compute the conditional expectation of the sample membership indicators $\E_p(I_iI_j|Q_n)$. We study this condition further and present an example in which we are able to compute it in Section~\ref{section:sim:data:I2}.

Conditions (C\ref{condition:kn}) and (C\ref{condition:kn:n}) state that $k_n$ grows to infinity, but not as fast as $n$ does.

\section{Consistency and Rate of Convergence}\label{section:consistency}

\begin{proposition}\label{proposition:consistency}
     Suppose that Conditions (C\ref{condition:y:bounded}) to (C\ref{condition:kn:n}) hold. The sample estimator $\widehat{m}_n(\xb)$ is $L^2$-consistent for $m(\xb)$ and satisfies
     \begin{itemize}
       \item[] If $d=1$,
                $$\E_{\xi}\E_{p}\left[\left\{\widehat{m}_n(\xb) - m(\xb)\right\}^2 \right] = O\left(\frac{1}{k_n} + \frac{k_n}{n} \right),$$
       \item[] and, if $d \geq 2$,
                $$\E_{\xi}\E_{p}\left[ \left\{\widehat{m}_n(\xb) - m(\xb)\right\}^2 \right] = O\left[\frac{1}{k_n} + \left(\frac{k_n}{n}\right)^{2/d} \right],$$
     \end{itemize}
    with $\xi$-probability one.
\end{proposition}

\begin{proof}
We can write
\begin{align}
  \E_{\xi}\E_{p}\left[ \left\{\widehat{m}_n(\xb) - m(\xb)\right\}^2\right]
    &= \E_{\xi}\E_{p}\left[ \left\{\widehat{m}_n(\xb) -\widehat{m}^*_n(\xb) + \widehat{m}^*_n(\xb)  - m(\xb)\right\}^2\right] \\
    &\leq 2\E_{\xi}\E_{p}\left[ \left\{\widehat{m}_n(\xb) -\widehat{m}^*_n(\xb)\right\}^2\right] +  2\E_{\xi}\E_{p}\left[ \left\{\widehat{m}^*_n(\xb)  - m(\xb)\right\}^2\right].
\end{align}
Proposition~\ref{proposition:knn:model:consistency:mstar} in the Appendix tells us that, when conditions (C\ref{condition:y:bounded}) to (C\ref{condition:density:neighborhoods}), and (C\ref{condition:non:informative}) hold, we have $\E_{\xi}\E_{p}\left[\left\{\widehat{m}^*_n(\xb) - m(\xb)\right\}^2 \right] = O\left(\frac{1}{k_n} + \frac{k_n}{n} \right)$ if $d=1$ and 
$\E_{\xi}\E_{p}\left[\left\{\widehat{m}^*_n(\xb) - m(\xb)\right\}^2 \right] = O\left[\frac{1}{k_n} + \left(\frac{k_n}{n}\right)^{2/d} \right]$ if $d \geq 2$.

Moreover, Proposition~\ref{proposition:knn:design:consistency} in the Appendix tells us that, when conditions (C\ref{condition:y:bounded}) to (C\ref{condition:I2}) hold, we have $ \E_{p}\left[ \left\{\widehat{m}^*_n(\xb) - \widehat{m}_n(\xb)\right\}^2\right]$ is $O(k_n^{-1})$ with $\xi$-probability one. As a result, \newline
$\E_{\xi}\E_{p}\left[ \left\{\widehat{m}^*_n(\xb) - \widehat{m}_n(\xb)\right\}^2\right]$ is $O(k_n^{-1})$.

We can conclude that, when conditions (C\ref{condition:y:bounded}) to (C\ref{condition:I2}) hold, if $d=1$,
                $$\E_{\xi}\E_{p}\left[\left\{\widehat{m}_n(\xb) - m(\xb)\right\}^2 \right] = O\left(\frac{1}{k_n} + \frac{k_n}{n} \right),$$
and, if $d \geq 2$,
                $$\E_{\xi}\E_{p}\left[ \left\{\widehat{m}_n(\xb) - m(\xb)\right\}^2 \right] = O\left[\frac{1}{k_n} + \left(\frac{k_n}{n}\right)^{2/d} \right].$$
If, moreover, conditions~(C\ref{condition:kn}) and (C\ref{condition:kn:n}) are satisfied, we can conclude that estimator $\widehat{m}_n(\xb)$ is $L^2$-consistent for $m(\xb)$.
\end{proof}

We can see that the nearest neighbor estimate based on survey data suffers from the curse of dimensionality \citep{bel:61,sto:85}. The lower bound for the rate of convergence decreases as the number of covariates $d$ increases. It is also the case in the case of independent and identically distributed observations, see \cite{bia:dev:15:nn} page 191.

\begin{corollary}\label{corollary}
       Suppose that Conditions (C\ref{condition:y:bounded}) to (C\ref{condition:I2}) hold.
     \begin{itemize}
       \item[] If $d=1$ and $k_n = n^\frac{1}{2}$,
                $$\E_{\xi}\E_{p}\left[\left\{\widehat{m}_n(\xb) - m(\xb)\right\}^2 \right] = O\left(n^{-\frac{1}{2}} \right),$$
        \item[] If $d \geq 2$ and $k_n =M_5^{\frac{d}{d+2}}n^{\frac{2}{d+2}}$ for a positive constant $M_5$ not related to $d$,
                $$\E_{\xi}\E_{p}\left[ \left\{\widehat{m}_n(\xb) - m(\xb)\right\}^2 \right] = O\left(n^{-\frac{2}{d+2}}\right).$$
     \end{itemize}
\end{corollary}

Note that Conditions (C\ref{condition:kn}) and (C\ref{condition:kn:n}) hold for these choices of $k_n$. Corollary~\ref{corollary} is a good example of the curse of dimensionality. Indeed, the lower bound for the rate of convergence increases as the number of covariates $d$ increases. If the goal is to find the minimum sample size to ensure that the error $\E_{\xi}\E_{p}\left[ \left\{\widehat{m}_n(\xb) - m(\xb)\right\}^2 \right]$ is smaller than a preset value $\epsilon$, Corollary~\ref{corollary} allows us to show that if $d \geq 2$ then the sample size is of the order of $(1/\epsilon)^{\frac{d+2}{2}}$. This quantity grows exponentially in $d$.

\section{Empirical Studies}\label{section:empirical}

\subsection{Simulated Data: Condition (C\ref{condition:density:neighborhoods})}\label{section:sim:data:condition:density}
In this section, we present cases in which Condition (C\ref{condition:density:neighborhoods}) is verified or violated based on simulated data. We generate nine embedded populations $U_v, v = 1, \ldots, 9$ of respective size $N = 50, 100, 200, 500, 1000, 5000, 10000, 20000, 50000$. We generate the values $x_i, i \in U_9$ of one covariate $X$ as i.i.d. draws of a random variables distributed as a uniform distribution with minimum 0 and maximum 1. For each population $U_v, v = 1, \ldots, 9$, we select three samples $S_v^i, i = \mbox{pps}, \mbox{srswor}, \mbox{stratified}$ of size $n = 0.4N$ with proportional to size sampling (pps), simple random sampling without replacement (srswor), and stratified sampling (stratified), respectively. The sampling fraction is $f = 0.4$. For proportional to size sampling, the inclusion probabilities are proportional to the values of covariate $X$. A unit with a larger value of covariate $X$ is more likely to be included in the sample than a unit with a smaller value. For stratified sampling, the population is stratified into four strata defined based on the quartiles of the values of covariate $X$. A sample of respective size $0.1n$, $0.2n$, $0.2n$, and $0.5n$ is selected from each of the four stratum with srswor. The sampling fraction within the strata is of approximately 0.16, 0.32, 0.32, and 0.8, respectively. All units within a stratum are equally likely to be included in the sample. The higher the values of covariate $X$ within a stratum, the higher the sampling fraction within this stratum. We consider $k_n = \lfloor n^{1/2} \rfloor$ where $\lfloor \cdot \rfloor$ is the floor function that rounds to the nearest smaller integer. We generate a sequence of values in the support $X$ equally spaced from 0 to 1 with increment of $0.02$. That is, we generate a vector $\xb^o = \left(0, 0.02, 0.04, \ldots, 1\right)^\top$. Then, for each value $x^o_i$ of the sequence, we compute the ratio between the overall sampling fraction $f = 0.4$ to the sampling fraction in the ball of radius $\rho_{k_nS}(x^o_i)$ centered at $x^o_i$. This ratio corresponds to the quantity in the left-hand side of Equation~\eqref{equation:density:neighborhoods}. Figure~\ref{fig:density:neighborhoods} shows the maximum over all values in $\xb^o$ of this ratio for three sampling designs and different population sizes. We can see that the ratio is bounded for simple random sampling without replacement and stratified sampling. Condition (C\ref{condition:density:neighborhoods}) is verified for these two sampling designs. However, the ratio increases with the population and sample size for proportional to size sampling. Condition (C\ref{condition:density:neighborhoods}) is violated for this sampling design. 

\begin{figure}
\centering
\begin{tikzpicture}[x=1pt,y=1pt]
\definecolor{fillColor}{RGB}{255,255,255}
\path[use as bounding box,fill=fillColor,fill opacity=0.00] (0,0) rectangle (361.35,180.67);
\begin{scope}
\path[clip] (  0.00,  0.00) rectangle (361.35,180.67);
\definecolor{drawColor}{RGB}{255,255,255}
\definecolor{fillColor}{RGB}{255,255,255}

\path[draw=drawColor,line width= 0.6pt,line join=round,line cap=round,fill=fillColor] (  0.00,  0.00) rectangle (361.35,180.68);
\end{scope}
\begin{scope}
\path[clip] ( 31.71, 30.69) rectangle (278.83,175.17);
\definecolor{fillColor}{RGB}{255,255,255}

\path[fill=fillColor] ( 31.71, 30.69) rectangle (278.83,175.18);
\definecolor{drawColor}{gray}{0.92}

\path[draw=drawColor,line width= 0.3pt,line join=round] ( 31.71, 41.49) --
	(278.83, 41.49);

\path[draw=drawColor,line width= 0.3pt,line join=round] ( 31.71, 78.64) --
	(278.83, 78.64);

\path[draw=drawColor,line width= 0.3pt,line join=round] ( 31.71,115.78) --
	(278.83,115.78);

\path[draw=drawColor,line width= 0.3pt,line join=round] ( 31.71,152.92) --
	(278.83,152.92);

\path[draw=drawColor,line width= 0.3pt,line join=round] ( 65.21, 30.69) --
	( 65.21,175.17);

\path[draw=drawColor,line width= 0.3pt,line join=round] (110.18, 30.69) --
	(110.18,175.17);

\path[draw=drawColor,line width= 0.3pt,line join=round] (155.16, 30.69) --
	(155.16,175.17);

\path[draw=drawColor,line width= 0.3pt,line join=round] (200.13, 30.69) --
	(200.13,175.17);

\path[draw=drawColor,line width= 0.3pt,line join=round] (245.11, 30.69) --
	(245.11,175.17);

\path[draw=drawColor,line width= 0.6pt,line join=round] ( 31.71, 60.06) --
	(278.83, 60.06);

\path[draw=drawColor,line width= 0.6pt,line join=round] ( 31.71, 97.21) --
	(278.83, 97.21);

\path[draw=drawColor,line width= 0.6pt,line join=round] ( 31.71,134.35) --
	(278.83,134.35);

\path[draw=drawColor,line width= 0.6pt,line join=round] ( 31.71,171.49) --
	(278.83,171.49);

\path[draw=drawColor,line width= 0.6pt,line join=round] ( 42.72, 30.69) --
	( 42.72,175.17);

\path[draw=drawColor,line width= 0.6pt,line join=round] ( 87.70, 30.69) --
	( 87.70,175.17);

\path[draw=drawColor,line width= 0.6pt,line join=round] (132.67, 30.69) --
	(132.67,175.17);

\path[draw=drawColor,line width= 0.6pt,line join=round] (177.65, 30.69) --
	(177.65,175.17);

\path[draw=drawColor,line width= 0.6pt,line join=round] (222.62, 30.69) --
	(222.62,175.17);

\path[draw=drawColor,line width= 0.6pt,line join=round] (267.60, 30.69) --
	(267.60,175.17);
\definecolor{drawColor}{RGB}{0,0,0}

\path[draw=drawColor,line width= 0.6pt,line join=round] ( 42.94, 50.16) --
	( 43.17, 56.76) --
	( 43.62, 52.64) --
	( 44.97, 59.71) --
	( 47.22, 79.63) --
	( 65.21, 96.87) --
	( 87.70,107.03) --
	(132.67,129.76) --
	(267.60,168.61);

\path[draw=drawColor,line width= 0.6pt,dash pattern=on 2pt off 2pt ,line join=round] ( 42.94, 39.02) --
	( 43.17, 46.03) --
	( 43.62, 47.07) --
	( 44.97, 42.38) --
	( 47.22, 38.52) --
	( 65.21, 40.03) --
	( 87.70, 38.57) --
	(132.67, 38.11) --
	(267.60, 37.25);

\path[draw=drawColor,line width= 0.6pt,dash pattern=on 4pt off 2pt ,line join=round] ( 42.94, 50.16) --
	( 43.17, 50.99) --
	( 43.62, 52.64) --
	( 44.97, 58.65) --
	( 47.22, 58.33) --
	( 65.21, 57.81) --
	( 87.70, 60.10) --
	(132.67, 59.81) --
	(267.60, 58.92);
\definecolor{drawColor}{gray}{0.20}

\path[draw=drawColor,line width= 0.6pt,line join=round,line cap=round] ( 31.71, 30.69) rectangle (278.83,175.18);
\end{scope}
\begin{scope}
\path[clip] (  0.00,  0.00) rectangle (361.35,180.67);
\definecolor{drawColor}{gray}{0.30}

\node[text=drawColor,anchor=base east,inner sep=0pt, outer sep=0pt, scale=  0.88] at ( 26.76, 57.03) {3};

\node[text=drawColor,anchor=base east,inner sep=0pt, outer sep=0pt, scale=  0.88] at ( 26.76, 94.18) {6};

\node[text=drawColor,anchor=base east,inner sep=0pt, outer sep=0pt, scale=  0.88] at ( 26.76,131.32) {9};

\node[text=drawColor,anchor=base east,inner sep=0pt, outer sep=0pt, scale=  0.88] at ( 26.76,168.46) {12};
\end{scope}
\begin{scope}
\path[clip] (  0.00,  0.00) rectangle (361.35,180.67);
\definecolor{drawColor}{gray}{0.20}

\path[draw=drawColor,line width= 0.6pt,line join=round] ( 28.96, 60.06) --
	( 31.71, 60.06);

\path[draw=drawColor,line width= 0.6pt,line join=round] ( 28.96, 97.21) --
	( 31.71, 97.21);

\path[draw=drawColor,line width= 0.6pt,line join=round] ( 28.96,134.35) --
	( 31.71,134.35);

\path[draw=drawColor,line width= 0.6pt,line join=round] ( 28.96,171.49) --
	( 31.71,171.49);
\end{scope}
\begin{scope}
\path[clip] (  0.00,  0.00) rectangle (361.35,180.67);
\definecolor{drawColor}{gray}{0.20}

\path[draw=drawColor,line width= 0.6pt,line join=round] ( 42.72, 27.94) --
	( 42.72, 30.69);

\path[draw=drawColor,line width= 0.6pt,line join=round] ( 87.70, 27.94) --
	( 87.70, 30.69);

\path[draw=drawColor,line width= 0.6pt,line join=round] (132.67, 27.94) --
	(132.67, 30.69);

\path[draw=drawColor,line width= 0.6pt,line join=round] (177.65, 27.94) --
	(177.65, 30.69);

\path[draw=drawColor,line width= 0.6pt,line join=round] (222.62, 27.94) --
	(222.62, 30.69);

\path[draw=drawColor,line width= 0.6pt,line join=round] (267.60, 27.94) --
	(267.60, 30.69);
\end{scope}
\begin{scope}
\path[clip] (  0.00,  0.00) rectangle (361.35,180.67);
\definecolor{drawColor}{gray}{0.30}

\node[text=drawColor,anchor=base,inner sep=0pt, outer sep=0pt, scale=  0.88] at ( 42.72, 19.68) {0};

\node[text=drawColor,anchor=base,inner sep=0pt, outer sep=0pt, scale=  0.88] at ( 87.70, 19.68) {10000};

\node[text=drawColor,anchor=base,inner sep=0pt, outer sep=0pt, scale=  0.88] at (132.67, 19.68) {20000};

\node[text=drawColor,anchor=base,inner sep=0pt, outer sep=0pt, scale=  0.88] at (177.65, 19.68) {30000};

\node[text=drawColor,anchor=base,inner sep=0pt, outer sep=0pt, scale=  0.88] at (222.62, 19.68) {40000};

\node[text=drawColor,anchor=base,inner sep=0pt, outer sep=0pt, scale=  0.88] at (267.60, 19.68) {50000};
\end{scope}
\begin{scope}
\path[clip] (  0.00,  0.00) rectangle (361.35,180.67);
\definecolor{drawColor}{RGB}{0,0,0}

\node[text=drawColor,anchor=base,inner sep=0pt, outer sep=0pt, scale=  1.10] at (155.27,  7.64) {N};
\end{scope}
\begin{scope}
\path[clip] (  0.00,  0.00) rectangle (361.35,180.67);
\definecolor{drawColor}{RGB}{0,0,0}

\node[text=drawColor,rotate= 90.00,anchor=base,inner sep=0pt, outer sep=0pt, scale=  1.10] at ( 13.08,102.93) {Max ratio $f / f(x_i^o)$ };
\end{scope}
\begin{scope}
\path[clip] (  0.00,  0.00) rectangle (361.35,180.67);
\definecolor{fillColor}{RGB}{255,255,255}

\path[fill=fillColor] (289.83, 68.58) rectangle (355.85,137.28);
\end{scope}
\begin{scope}
\path[clip] (  0.00,  0.00) rectangle (361.35,180.67);
\definecolor{drawColor}{RGB}{0,0,0}

\node[text=drawColor,anchor=base west,inner sep=0pt, outer sep=0pt, scale=  1.00] at (295.33,123.92) {N};
\end{scope}
\begin{scope}
\path[clip] (  0.00,  0.00) rectangle (361.35,180.67);
\definecolor{fillColor}{RGB}{255,255,255}

\path[fill=fillColor] (295.33,102.99) rectangle (309.78,117.45);
\end{scope}
\begin{scope}
\path[clip] (  0.00,  0.00) rectangle (361.35,180.67);
\definecolor{drawColor}{RGB}{0,0,0}

\path[draw=drawColor,line width= 0.6pt,line join=round] (296.78,110.22) -- (308.34,110.22);
\end{scope}
\begin{scope}
\path[clip] (  0.00,  0.00) rectangle (361.35,180.67);
\definecolor{fillColor}{RGB}{255,255,255}

\path[fill=fillColor] (295.33, 88.54) rectangle (309.78,102.99);
\end{scope}
\begin{scope}
\path[clip] (  0.00,  0.00) rectangle (361.35,180.67);
\definecolor{drawColor}{RGB}{0,0,0}

\path[draw=drawColor,line width= 0.6pt,dash pattern=on 2pt off 2pt ,line join=round] (296.78, 95.76) -- (308.34, 95.76);
\end{scope}
\begin{scope}
\path[clip] (  0.00,  0.00) rectangle (361.35,180.67);
\definecolor{fillColor}{RGB}{255,255,255}

\path[fill=fillColor] (295.33, 74.08) rectangle (309.78, 88.54);
\end{scope}
\begin{scope}
\path[clip] (  0.00,  0.00) rectangle (361.35,180.67);
\definecolor{drawColor}{RGB}{0,0,0}

\path[draw=drawColor,line width= 0.6pt,dash pattern=on 4pt off 2pt ,line join=round] (296.78, 81.31) -- (308.34, 81.31);
\end{scope}
\begin{scope}
\path[clip] (  0.00,  0.00) rectangle (361.35,180.67);
\definecolor{drawColor}{RGB}{0,0,0}

\node[text=drawColor,anchor=base west,inner sep=0pt, outer sep=0pt, scale=  0.90] at (315.28,107.12) {pps};
\end{scope}
\begin{scope}
\path[clip] (  0.00,  0.00) rectangle (361.35,180.67);
\definecolor{drawColor}{RGB}{0,0,0}

\node[text=drawColor,anchor=base west,inner sep=0pt, outer sep=0pt, scale=  0.90] at (315.28, 92.67) {srswor};
\end{scope}
\begin{scope}
\path[clip] (  0.00,  0.00) rectangle (361.35,180.67);
\definecolor{drawColor}{RGB}{0,0,0}

\node[text=drawColor,anchor=base west,inner sep=0pt, outer sep=0pt, scale=  0.90] at (315.28, 78.21) {stratified};
\end{scope}
\end{tikzpicture}
\caption{Simulated data: maximum value of the ratio of the overall sampling fraction $f$ to the sampling fraction $f(x^o_i)$ in the ball of radius $\rho_{k_nS}(x^o_i)$ centered at $x^o_i$ over a sequence of values $\xb^o$ in the support of the covariate. Different population sizes $N$ and sampling designs are considered. The sampling designs are proportional to size sampling (pps), simple random sampling without replacement (srswor), and stratified sampling (stratified).} 
\label{fig:density:neighborhoods}
\end{figure}
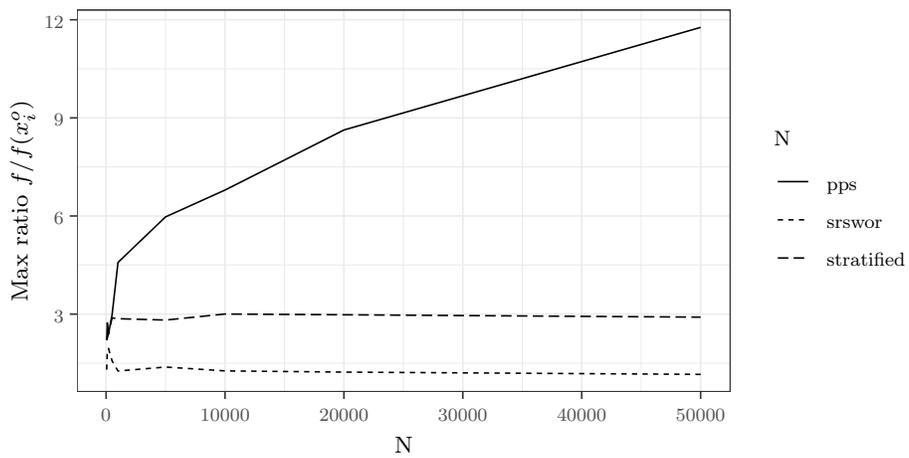

\subsection{Simulated Data: Condition (C\ref{condition:I2})}\label{section:sim:data:I2}
Condition (C\ref{condition:I2}) is complicated to study empirically because it is often impossible to compute the conditional expectation of the sample membership indicators $\E_p(I_iI_j|Q_n)$. We present here a limited example in which we are able to compute this expectation. We generate nine embedded populations $U_v, v = 1, \ldots, 9$ of respective size $N$ ranging from 10 to 50 in increments of 5. We generate the values $x_i, i \in U_9$ of one covariate $X$ as i.i.d. draws of a random variables distributed as a uniform distribution with minimum 0 and maximum 1. For each population $U_v, v = 1, \ldots, 9$, we select a sample of size $n = 0.4N$ using simple random sampling without replacement. The sampling fraction is $f = 0.4$. We consider $k_n = \lfloor n^{1/2} \rfloor$. Then we compute the conditional expectation of the sample membership indicators $\E_p(I_iI_j|Q_n)$ and the values of $r_{ij} = \E_p(I_iI_j|Q_n) - \pi_{ij}$ for all $i,j \in U_v$. The average absolute value of $r_{ij}, i,j\in U_v$ for the nine generated populations $U_v, v = 1, \ldots, 9$ is shown in Figure~\ref{fig:rij:abs}. We see that the absolute value of $r_{ij}$ decreases as the sample size increases. This seems to indicate that Condition (C\ref{condition:I2}) is verified or at least does not provide evidence that this condition fails to hold.

\begin{figure}
\centering
\begin{tikzpicture}[x=1pt,y=1pt]
\definecolor{fillColor}{RGB}{255,255,255}
\path[use as bounding box,fill=fillColor,fill opacity=0.00] (0,0) rectangle (361.35,180.67);
\begin{scope}
\path[clip] (  0.00,  0.00) rectangle (361.35,180.67);
\definecolor{drawColor}{RGB}{255,255,255}
\definecolor{fillColor}{RGB}{255,255,255}

\path[draw=drawColor,line width= 0.6pt,line join=round,line cap=round,fill=fillColor] (  0.00,  0.00) rectangle (361.35,180.68);
\end{scope}
\begin{scope}
\path[clip] ( 42.95, 30.69) rectangle (355.85,175.17);
\definecolor{fillColor}{RGB}{255,255,255}

\path[fill=fillColor] ( 42.95, 30.69) rectangle (355.85,175.18);
\definecolor{drawColor}{gray}{0.92}

\path[draw=drawColor,line width= 0.3pt,line join=round] ( 42.95, 45.60) --
	(355.85, 45.60);

\path[draw=drawColor,line width= 0.3pt,line join=round] ( 42.95, 91.46) --
	(355.85, 91.46);

\path[draw=drawColor,line width= 0.3pt,line join=round] ( 42.95,137.33) --
	(355.85,137.33);

\path[draw=drawColor,line width= 0.3pt,line join=round] ( 92.73, 30.69) --
	( 92.73,175.17);

\path[draw=drawColor,line width= 0.3pt,line join=round] (163.85, 30.69) --
	(163.85,175.17);

\path[draw=drawColor,line width= 0.3pt,line join=round] (234.96, 30.69) --
	(234.96,175.17);

\path[draw=drawColor,line width= 0.3pt,line join=round] (306.07, 30.69) --
	(306.07,175.17);

\path[draw=drawColor,line width= 0.6pt,line join=round] ( 42.95, 68.53) --
	(355.85, 68.53);

\path[draw=drawColor,line width= 0.6pt,line join=round] ( 42.95,114.40) --
	(355.85,114.40);

\path[draw=drawColor,line width= 0.6pt,line join=round] ( 42.95,160.27) --
	(355.85,160.27);

\path[draw=drawColor,line width= 0.6pt,line join=round] ( 57.18, 30.69) --
	( 57.18,175.17);

\path[draw=drawColor,line width= 0.6pt,line join=round] (128.29, 30.69) --
	(128.29,175.17);

\path[draw=drawColor,line width= 0.6pt,line join=round] (199.40, 30.69) --
	(199.40,175.17);

\path[draw=drawColor,line width= 0.6pt,line join=round] (270.51, 30.69) --
	(270.51,175.17);

\path[draw=drawColor,line width= 0.6pt,line join=round] (341.63, 30.69) --
	(341.63,175.17);
\definecolor{drawColor}{RGB}{0,0,0}

\path[draw=drawColor,line width= 0.6pt,line join=round] ( 57.18,168.61) --
	( 92.73,120.95) --
	(128.29, 93.36) --
	(163.85, 75.59) --
	(199.40, 63.22) --
	(234.96, 54.14) --
	(270.51, 47.19) --
	(306.07, 41.70) --
	(341.63, 37.25);
\definecolor{drawColor}{gray}{0.20}

\path[draw=drawColor,line width= 0.6pt,line join=round,line cap=round] ( 42.95, 30.69) rectangle (355.85,175.18);
\end{scope}
\begin{scope}
\path[clip] (  0.00,  0.00) rectangle (361.35,180.67);
\definecolor{drawColor}{gray}{0.30}

\node[text=drawColor,anchor=base east,inner sep=0pt, outer sep=0pt, scale=  0.88] at ( 38.00, 65.50) {0.272};

\node[text=drawColor,anchor=base east,inner sep=0pt, outer sep=0pt, scale=  0.88] at ( 38.00,111.37) {0.274};

\node[text=drawColor,anchor=base east,inner sep=0pt, outer sep=0pt, scale=  0.88] at ( 38.00,157.24) {0.276};
\end{scope}
\begin{scope}
\path[clip] (  0.00,  0.00) rectangle (361.35,180.67);
\definecolor{drawColor}{gray}{0.20}

\path[draw=drawColor,line width= 0.6pt,line join=round] ( 40.20, 68.53) --
	( 42.95, 68.53);

\path[draw=drawColor,line width= 0.6pt,line join=round] ( 40.20,114.40) --
	( 42.95,114.40);

\path[draw=drawColor,line width= 0.6pt,line join=round] ( 40.20,160.27) --
	( 42.95,160.27);
\end{scope}
\begin{scope}
\path[clip] (  0.00,  0.00) rectangle (361.35,180.67);
\definecolor{drawColor}{gray}{0.20}

\path[draw=drawColor,line width= 0.6pt,line join=round] ( 57.18, 27.94) --
	( 57.18, 30.69);

\path[draw=drawColor,line width= 0.6pt,line join=round] (128.29, 27.94) --
	(128.29, 30.69);

\path[draw=drawColor,line width= 0.6pt,line join=round] (199.40, 27.94) --
	(199.40, 30.69);

\path[draw=drawColor,line width= 0.6pt,line join=round] (270.51, 27.94) --
	(270.51, 30.69);

\path[draw=drawColor,line width= 0.6pt,line join=round] (341.63, 27.94) --
	(341.63, 30.69);
\end{scope}
\begin{scope}
\path[clip] (  0.00,  0.00) rectangle (361.35,180.67);
\definecolor{drawColor}{gray}{0.30}

\node[text=drawColor,anchor=base,inner sep=0pt, outer sep=0pt, scale=  0.88] at ( 57.18, 19.68) {10};

\node[text=drawColor,anchor=base,inner sep=0pt, outer sep=0pt, scale=  0.88] at (128.29, 19.68) {20};

\node[text=drawColor,anchor=base,inner sep=0pt, outer sep=0pt, scale=  0.88] at (199.40, 19.68) {30};

\node[text=drawColor,anchor=base,inner sep=0pt, outer sep=0pt, scale=  0.88] at (270.51, 19.68) {40};

\node[text=drawColor,anchor=base,inner sep=0pt, outer sep=0pt, scale=  0.88] at (341.63, 19.68) {50};
\end{scope}
\begin{scope}
\path[clip] (  0.00,  0.00) rectangle (361.35,180.67);
\definecolor{drawColor}{RGB}{0,0,0}

\node[text=drawColor,anchor=base,inner sep=0pt, outer sep=0pt, scale=  1.10] at (199.40,  7.64) {N};
\end{scope}
\begin{scope}
\path[clip] (  0.00,  0.00) rectangle (361.35,180.67);
\definecolor{drawColor}{RGB}{0,0,0}

\node[text=drawColor,rotate= 90.00,anchor=base,inner sep=0pt, outer sep=0pt, scale=  1.10] at ( 13.08,102.93) {Value};
\end{scope}
\end{tikzpicture}
\caption{Simulated data: average value of $|r_{ij}|$ for different population sizes N.}
\label{fig:rij:abs}
\end{figure}
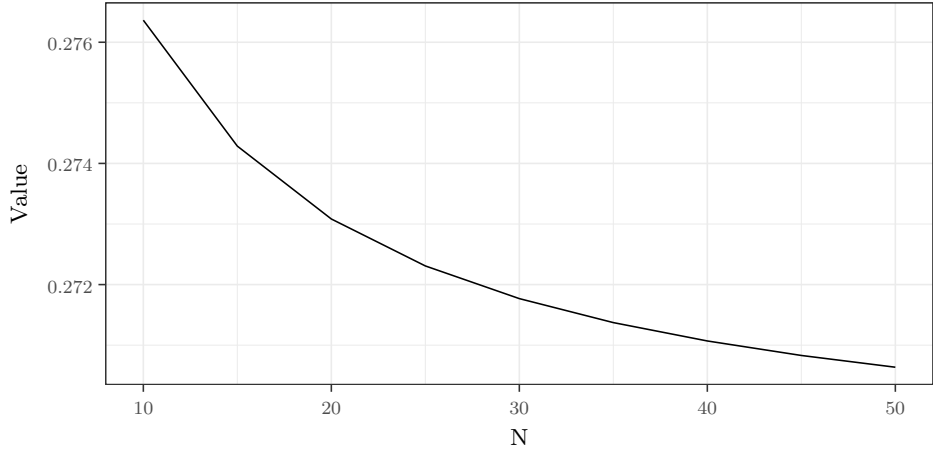

\subsection{Simulated Data: Asymptotics and Consistency}\label{section:sim:data:consistency}

We study the asymptotic behavior of the sample $k_n$-nearest neighbors estimator $\widehat{m}_n$ defined in Equation~\eqref{estimator:mn} and illustrate the result stated in Proposition~\ref{proposition:consistency} using a Monte Carlo simulation study on simulated data. We generate nine embedded populations $U_v, v = 1, \ldots, 9$ of respective size $N = 50, 100, 200, 500, 1000, 5000,$ $10000, 20000, 50000$. We generate the values $x_i, i \in U_9$ of one covariate $X$ as i.i.d. draws of a random variable distributed as a uniform distribution with minimum 0 and maximum 1. Then we generate the values $y_i, i \in U_9$ of the response as follows
\begin{align}
  y_i &= m(x_i) + \eps_i,
\end{align}
where $m(x) = 2x + \sin(2\cdot3.14\cdot x)$ and $\eps_i, i \in U_9$, are draws of i.i.d. random variables distributed as a normal distribution with mean 0 and standard deviation 0.5. Then we create a vector $\xb^o$ of 10 equally spaced values from the smallest to the largest value of $x_i, i \in U_9$. These values lie in the support of the covariate and are the points at which we will estimate regression function $m$. 

Then we run $L = 1000$ simulations for each population $U_v$ as explained in what follows. For a simulation run $\ell= 1, \ldots, L$, we select a sample of size $n = 0.2N$ using simple random sampling without replacement. We set $k_n = \lfloor n^{1/2}\rfloor$. Then we compute the sample $k_n$-nearest neighbors estimator $\widehat{m}_n(x^o_i,\ell)$ defined in Equation~\eqref{estimator:mn} for each of the ten values $x^o_i, i \in 1, \ldots 10$. At the end of the simulation runs, we have computed $L = 1000$ estimates $\widehat{m}_n (x^o_i,\ell)$ of $m(x^o_i)$ for each $x^o_i, i = 1, \ldots, 10$.

For each $x^o_i, i = 1, \ldots, 10$, we compute the Monte Carlo Mean Squared Error (MSE) of $\widehat{m}_n(x^o_i)$
\begin{align}
  \frac{1}{L} \sum_{\ell = 1}^{L} \left\{ \widehat{m}_n (x^o_i,\ell) - m(x^o_i) \right\}^2.
\end{align} 
This is an estimate of $\E_{p}\left[\left\{\widehat{m}_n(x^o_i) - m(x^o_i)\right\}^2 \right]$. Figure~\ref{fig:sim:consistency} shows the boxplots of the ten values of the MSE of $\widehat{m}_n(x^o_i)$, $i = 1, \ldots, 10$, for all nine populations. The dashed line is $\frac{1}{k_n} + \frac{k_n}{n}$ multiplied by an adjustment factor of 2.2 to fit the scale of the boxplots. First, we can see that the MSE of $\widehat{m}_n$ converges towards 0 as the population size increases. Second, we can see that $\frac{1}{k_n} + \frac{k_n}{n}$ and the MSE of $\widehat{m}_n$ seem to have a rate of convergence of the same order. These elements illustrate Proposition~\ref{proposition:consistency}.  

\begin{figure}
\centering
\input{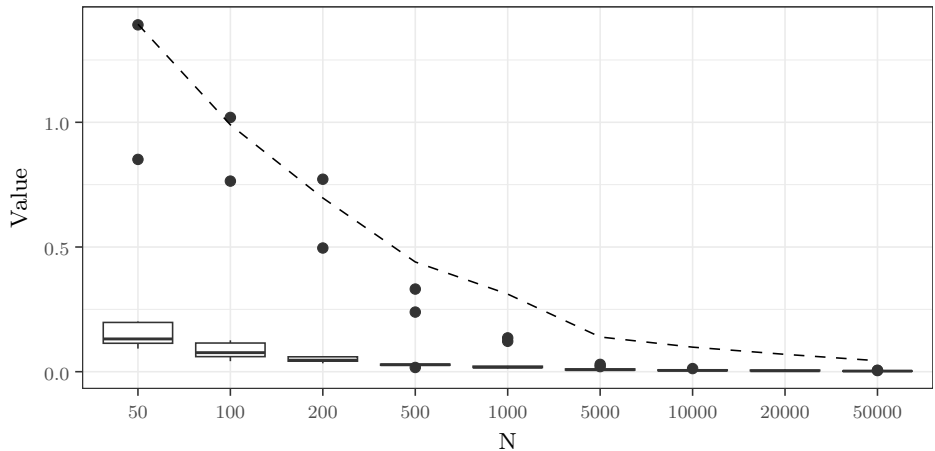}
\caption{Simulated data: boxplots of the value of the MSE of the sample $k_n$-nearest neighbors estimator $\widehat{m}_n(x^o_i)$, $i = 1, \ldots, 10$ for nine populations of respective size N. The dashed line is the value of $\frac{1}{k_n} + \frac{k_n}{n}$ for the nine populations (multiplied by 2.2 for graphical reasons).}
\label{fig:sim:consistency}
\end{figure}

\subsection{Wine Data: Asymptotics and Consistency}\label{section:wine:data}

We consider the White Wine data \citep{wine_quality} available here \newline 
\href{https://archive.ics.uci.edu/dataset/186/wine+quality}{https://archive.ics.uci.edu/dataset/186/wine+quality}.\newline
The dataset contains information about 4898 white variants of the Portuguese "Vinho Verde" wine. The goal is to predict wine quality based on physicochemical tests. The wine quality is determined via a score from 0 (very bad) to 10 (excellent) provided by at least three sensory assessors. The potential predictors include 11 variables based on physicochemical tests, namely fixed acidity, volatile acidity, citric acid, residual sugar, chlorides, free sulfur dioxide, total sulfur dioxide, density, pH, sulphates, and alcohol. An inspection of the dataset reveals some high multicollinearity between the predictors. 
For instance, the correlation between residual sugar and density is 0.84, that between density and alcohol is -0.78. The variance inflation factor of the linear regression of the quality on the 11 physicochemical measurements is about 28 for variable density. This value is very high and indicates the presence of substantial multicollinearity. The variance inflation factors of the linear regression of the quality on the 10 physicochemical measurements excluding density are between 1.06 and 2.15, which seems to indicate that there is no multicollinearity when density is removed. As a result, we exclude density and predict the wine quality based on the 10 remaining physicochemical measurements (10 covariates). 

We generate five embedded populations $U_v, v = 1, \ldots, 5$ containing the first $N = 100, 500$, $1000, 2000, 4898$ wine variants of the White Wine data, respectively. With true data, regression function $m$ is unknown. We estimate it using the finite population $k$ nearest neighbor estimator $\widehat{m}_N$ defined in Equation~\eqref{estimator:mN} applied to the population of 4898 wine variants. We know that this estimator is $L^p$-consistent for regression function $m$ under mild conditions. We set $k_n = \lfloor n^{1/2}\rfloor$, where $n = \lfloor 0.2N \rfloor$ is the sample size selected in population $U_v$, see below. 

We create a set of values in the support of the covariates at which we will estimate the regression function as explained in what follows. For each covariate $X_i$, we consider three values in the grid: the minimum of the observed values for that covariate, the maximum, and the mid-point. A grid of values in the support of the covariates is obtained by considering all possible combinations of these three values per covariate. This creates $3^{10} = 59049$ vectors of values in the support of the covariates. For computational reasons, we select at random 100 of these vectors. We denote by $\xb^o_i$ the values of the covariates for the $i$-th vector, $i = 1, \ldots, 100$. We estimate the regression function at these 100 vectors of values.   

We run $L = 1000$ simulations as explained in Section~\ref{section:sim:data:consistency}. For each $\xb^o_i, i = 1, \ldots, 100$, we compute the Monte Carlo Mean Squared Error (MSE) of $\widehat{m}_n(\xb^o_i)$
\begin{align}
  \frac{1}{L} \sum_{\ell = 1}^{L} \left\{ \widehat{m}_n (\xb^o_i,\ell) - \widehat{m}_N(\xb^o_i) \right\}^2.
\end{align} 
This is an estimate of $\E_{p}\left[\left\{\widehat{m}_n(\xb^o_i) - m(\xb^o_i)\right\}^2 \right]$. Figure~\ref{fig:wn:consistency} shows the boxplots of the 100 values of the MSE of $\widehat{m}_n(\xb^o_i)$, $i = 1, \ldots, 100$, for all five populations. The dashed line is $\frac{1}{k_n} + \left(\frac{k_n}{n}\right)^{2/10}$ multiplied by an adjustment factor of 4.5 to fit the scale of the boxplots. First, we can see that the MSE of $\widehat{m}_n$ converges towards 0 as the population size increases. Second, we can see that $\frac{1}{k_n} + \left(\frac{k_n}{n}\right)^{2/10}$ and the MSE of $\widehat{m}_n$ seem to have a rate of convergence of the same order. These elements illustrate Proposition~\ref{proposition:consistency}.

\begin{figure}
\centering
\begin{tikzpicture}[x=1pt,y=1pt]
\definecolor{fillColor}{RGB}{255,255,255}
\path[use as bounding box,fill=fillColor,fill opacity=0.00] (0,0) rectangle (361.35,180.67);
\begin{scope}
\path[clip] (  0.00,  0.00) rectangle (361.35,180.67);
\definecolor{drawColor}{RGB}{255,255,255}
\definecolor{fillColor}{RGB}{255,255,255}

\path[draw=drawColor,line width= 0.6pt,line join=round,line cap=round,fill=fillColor] (  0.00,  0.00) rectangle (361.35,180.68);
\end{scope}
\begin{scope}
\path[clip] ( 27.31, 30.69) rectangle (355.85,175.17);
\definecolor{fillColor}{RGB}{255,255,255}

\path[fill=fillColor] ( 27.31, 30.69) rectangle (355.85,175.18);
\definecolor{drawColor}{gray}{0.92}

\path[draw=drawColor,line width= 0.3pt,line join=round] ( 27.31, 51.81) --
	(355.85, 51.81);

\path[draw=drawColor,line width= 0.3pt,line join=round] ( 27.31, 81.86) --
	(355.85, 81.86);

\path[draw=drawColor,line width= 0.3pt,line join=round] ( 27.31,111.91) --
	(355.85,111.91);

\path[draw=drawColor,line width= 0.3pt,line join=round] ( 27.31,141.97) --
	(355.85,141.97);

\path[draw=drawColor,line width= 0.3pt,line join=round] ( 27.31,172.02) --
	(355.85,172.02);

\path[draw=drawColor,line width= 0.6pt,line join=round] ( 27.31, 36.78) --
	(355.85, 36.78);

\path[draw=drawColor,line width= 0.6pt,line join=round] ( 27.31, 66.83) --
	(355.85, 66.83);

\path[draw=drawColor,line width= 0.6pt,line join=round] ( 27.31, 96.89) --
	(355.85, 96.89);

\path[draw=drawColor,line width= 0.6pt,line join=round] ( 27.31,126.94) --
	(355.85,126.94);

\path[draw=drawColor,line width= 0.6pt,line join=round] ( 27.31,156.99) --
	(355.85,156.99);

\path[draw=drawColor,line width= 0.6pt,line join=round] ( 65.22, 30.69) --
	( 65.22,175.17);

\path[draw=drawColor,line width= 0.6pt,line join=round] (128.40, 30.69) --
	(128.40,175.17);

\path[draw=drawColor,line width= 0.6pt,line join=round] (191.58, 30.69) --
	(191.58,175.17);

\path[draw=drawColor,line width= 0.6pt,line join=round] (254.76, 30.69) --
	(254.76,175.17);

\path[draw=drawColor,line width= 0.6pt,line join=round] (317.94, 30.69) --
	(317.94,175.17);
\definecolor{drawColor}{gray}{0.20}

\path[draw=drawColor,line width= 0.6pt,line join=round] ( 65.22,143.43) -- ( 65.22,146.83);

\path[draw=drawColor,line width= 0.6pt,line join=round] ( 65.22, 46.34) -- ( 65.22, 38.83);

\path[draw=drawColor,line width= 0.6pt,fill=fillColor] ( 41.53,143.43) --
	( 41.53, 46.34) --
	( 88.91, 46.34) --
	( 88.91,143.43) --
	( 41.53,143.43) --
	cycle;

\path[draw=drawColor,line width= 1.1pt] ( 41.53,115.56) -- ( 88.91,115.56);

\path[draw=drawColor,line width= 0.6pt,line join=round] (128.40, 68.63) -- (128.40, 76.48);

\path[draw=drawColor,line width= 0.6pt,line join=round] (128.40, 42.28) -- (128.40, 38.41);

\path[draw=drawColor,line width= 0.6pt,fill=fillColor] (104.71, 68.63) --
	(104.71, 42.28) --
	(152.09, 42.28) --
	(152.09, 68.63) --
	(104.71, 68.63) --
	cycle;

\path[draw=drawColor,line width= 1.1pt] (104.71, 54.11) -- (152.09, 54.11);

\path[draw=drawColor,line width= 0.6pt,line join=round] (191.58, 56.16) -- (191.58, 60.35);

\path[draw=drawColor,line width= 0.6pt,line join=round] (191.58, 39.44) -- (191.58, 37.80);

\path[draw=drawColor,line width= 0.6pt,fill=fillColor] (167.89, 56.16) --
	(167.89, 39.44) --
	(215.27, 39.44) --
	(215.27, 56.16) --
	(167.89, 56.16) --
	cycle;

\path[draw=drawColor,line width= 1.1pt] (167.89, 41.70) -- (215.27, 41.70);

\path[draw=drawColor,line width= 0.6pt,line join=round] (254.76, 46.80) -- (254.76, 54.70);

\path[draw=drawColor,line width= 0.6pt,line join=round] (254.76, 38.20) -- (254.76, 37.88);

\path[draw=drawColor,line width= 0.6pt,fill=fillColor] (231.07, 46.80) --
	(231.07, 38.20) --
	(278.45, 38.20) --
	(278.45, 46.80) --
	(231.07, 46.80) --
	cycle;

\path[draw=drawColor,line width= 1.1pt] (231.07, 40.39) -- (278.45, 40.39);

\path[draw=drawColor,line width= 0.6pt,line join=round] (317.94, 40.45) -- (317.94, 43.37);

\path[draw=drawColor,line width= 0.6pt,line join=round] (317.94, 37.60) -- (317.94, 37.25);

\path[draw=drawColor,line width= 0.6pt,fill=fillColor] (294.25, 40.45) --
	(294.25, 37.60) --
	(341.63, 37.60) --
	(341.63, 40.45) --
	(294.25, 40.45) --
	cycle;

\path[draw=drawColor,line width= 1.1pt] (294.25, 38.88) -- (341.63, 38.88);
\definecolor{drawColor}{RGB}{0,0,0}

\path[draw=drawColor,line width= 0.6pt,dash pattern=on 4pt off 4pt ,line join=round] ( 65.22,168.61) --
	(128.40,135.63) --
	(191.58,125.90) --
	(254.76,117.83) --
	(317.94,108.93);
\definecolor{drawColor}{gray}{0.20}

\path[draw=drawColor,line width= 0.6pt,line join=round,line cap=round] ( 27.31, 30.69) rectangle (355.85,175.18);
\end{scope}
\begin{scope}
\path[clip] (  0.00,  0.00) rectangle (361.35,180.67);
\definecolor{drawColor}{gray}{0.30}

\node[text=drawColor,anchor=base east,inner sep=0pt, outer sep=0pt, scale=  0.88] at ( 22.36, 33.75) {0};

\node[text=drawColor,anchor=base east,inner sep=0pt, outer sep=0pt, scale=  0.88] at ( 22.36, 63.80) {1};

\node[text=drawColor,anchor=base east,inner sep=0pt, outer sep=0pt, scale=  0.88] at ( 22.36, 93.86) {2};

\node[text=drawColor,anchor=base east,inner sep=0pt, outer sep=0pt, scale=  0.88] at ( 22.36,123.91) {3};

\node[text=drawColor,anchor=base east,inner sep=0pt, outer sep=0pt, scale=  0.88] at ( 22.36,153.96) {4};
\end{scope}
\begin{scope}
\path[clip] (  0.00,  0.00) rectangle (361.35,180.67);
\definecolor{drawColor}{gray}{0.20}

\path[draw=drawColor,line width= 0.6pt,line join=round] ( 24.56, 36.78) --
	( 27.31, 36.78);

\path[draw=drawColor,line width= 0.6pt,line join=round] ( 24.56, 66.83) --
	( 27.31, 66.83);

\path[draw=drawColor,line width= 0.6pt,line join=round] ( 24.56, 96.89) --
	( 27.31, 96.89);

\path[draw=drawColor,line width= 0.6pt,line join=round] ( 24.56,126.94) --
	( 27.31,126.94);

\path[draw=drawColor,line width= 0.6pt,line join=round] ( 24.56,156.99) --
	( 27.31,156.99);
\end{scope}
\begin{scope}
\path[clip] (  0.00,  0.00) rectangle (361.35,180.67);
\definecolor{drawColor}{gray}{0.20}

\path[draw=drawColor,line width= 0.6pt,line join=round] ( 65.22, 27.94) --
	( 65.22, 30.69);

\path[draw=drawColor,line width= 0.6pt,line join=round] (128.40, 27.94) --
	(128.40, 30.69);

\path[draw=drawColor,line width= 0.6pt,line join=round] (191.58, 27.94) --
	(191.58, 30.69);

\path[draw=drawColor,line width= 0.6pt,line join=round] (254.76, 27.94) --
	(254.76, 30.69);

\path[draw=drawColor,line width= 0.6pt,line join=round] (317.94, 27.94) --
	(317.94, 30.69);
\end{scope}
\begin{scope}
\path[clip] (  0.00,  0.00) rectangle (361.35,180.67);
\definecolor{drawColor}{gray}{0.30}

\node[text=drawColor,anchor=base,inner sep=0pt, outer sep=0pt, scale=  0.88] at ( 65.22, 19.68) {100};

\node[text=drawColor,anchor=base,inner sep=0pt, outer sep=0pt, scale=  0.88] at (128.40, 19.68) {500};

\node[text=drawColor,anchor=base,inner sep=0pt, outer sep=0pt, scale=  0.88] at (191.58, 19.68) {1000};

\node[text=drawColor,anchor=base,inner sep=0pt, outer sep=0pt, scale=  0.88] at (254.76, 19.68) {2000};

\node[text=drawColor,anchor=base,inner sep=0pt, outer sep=0pt, scale=  0.88] at (317.94, 19.68) {4898};
\end{scope}
\begin{scope}
\path[clip] (  0.00,  0.00) rectangle (361.35,180.67);
\definecolor{drawColor}{RGB}{0,0,0}

\node[text=drawColor,anchor=base,inner sep=0pt, outer sep=0pt, scale=  1.10] at (191.58,  7.64) {N};
\end{scope}
\begin{scope}
\path[clip] (  0.00,  0.00) rectangle (361.35,180.67);
\definecolor{drawColor}{RGB}{0,0,0}

\node[text=drawColor,rotate= 90.00,anchor=base,inner sep=0pt, outer sep=0pt, scale=  1.10] at ( 13.08,102.93) {Value};
\end{scope}
\end{tikzpicture}
\caption{White Wine data: boxplots of the value of the MSE of the sample $k_n$-nearest neighbors estimator $\widehat{m}_n(\xb^o_i)$, $i = 1, \ldots, 100$ for five populations of respective size N. The dashed line is the value of $\frac{1}{k_n} + \left(\frac{k_n}{n}\right)^{2/10}$ for the five populations (multiplied by 4.5 for graphical reasons).}
\label{fig:wn:consistency}
\end{figure}

\section{Closing Remarks}\label{section:conclusion}


While consistency results for the $k$-NN regressor are well established for i.i.d. data, corresponding results for complex survey data are lacking. We adopt a superpopulation approach show that the $k$-NN regressor is $L^2$-consistent under conditions on the sampling design, the regression function, and the growth rate of the number of neighbors. We also establish lower bounds for the rate of convergence and show that the estimator suffers from the curse of dimensionality, as under i.i.d. data. Simulation studies and an application to real data illustrate the theoretical findings and support the predicted asymptotic behavior.

Several directions for future research follow from this work. First, the results rely on conditions that exclude some commonly used sampling designs, such as cluster sampling and systematic sampling. Extending the theory to cover these designs would be of interest. Second, the choice of the number of neighbors is treated asymptotically. Data-driven methods for selecting this number tailored to survey data represent an interesting avenue for future research. Finally, extending the results to related methods, such adaptive nearest neighbor estimators, or to classification problems, represent a direction of extension for future research.

\bibliography{bib} 
\bibliographystyle{apalike}

\appendix
\section*{Appendix: Propositions~\ref{proposition:knn:model:consistency:mstar} and \ref{proposition:knn:design:consistency}}\label{section:propositions}

Consider $Y_{(1)}(\xb),Y_{(2)}(\xb),\ldots,Y_{(N)}(\xb)$ the reordering of $Y_1,Y_2,\ldots,Y_N$ in increasing values of $\lVert \Xb - \xb \rVert$. The hypothetical estimator can be written
\begin{align}
  \widehat{m}^*_n(\xb)  &= \left[ \sum_{j \in U}  \1{j \in B\left(\xb, \rho_{k_nS}(\xb)\right)} \right]^{-1} \sum_{i \in U}  \1{i \in B\left(\xb, \rho_{k_nS}(\xb)\right)} Y_i\\
                        &= \sum_{i \in U} w_{i}Y_{(i)}(\xb),
\end{align}
where
\begin{align}
  w_{i} &= \left\{
              \begin{array}{ll}
                \left[ \sum_{j \in U}  \1{j \in B\left(\xb, \rho_{k_nS}(\xb)\right)} \right]^{-1}, & \hbox{if $i \leq \sum_{j \in U}  \1{j \in B\left(\xb, \rho_{k_nS}(\xb)\right)} $;} \\
                0, & \hbox{otherwise.}
              \end{array}
            \right.
\end{align}
Vector $\left( w_1, w_2, \ldots, w_N \right)^\top$ is a probability weight vector. That is, its components are larger than zero and sum up to 1. These weights are deterministic conditionally on $Q_n$.

\begin{proposition}\label{proposition:knn:model:consistency:mstar}
     Suppose that Conditions (C\ref{condition:y:bounded}) to (C\ref{condition:density:neighborhoods}), and (C\ref{condition:non:informative}) hold. The hypothetical estimator $\widehat{m}^*_n(\xb)$ satisfies the following:
     \begin{itemize}
       \item[] If $d=1$,
                $$\E_{\xi}\E_{p}\left[\left\{\widehat{m}^*_n(\xb) - m(\xb)\right\}^2 \right] \leq  \frac{\sigma^2}{k_n} + 8L^2 C \frac{k_n}{n},$$
        \item[] If $d \geq 2$,
                $$\E_{\xi}\E_{p}\left[\left\{\widehat{m}^*_n(\xb) - m(\xb)\right\}^2 \right] \leq \frac{\sigma^2}{k_n} + c_d L^2\left(C\frac{k_n}{n}\right)^{2/d},$$
     \end{itemize}
  where
                 \begin{align}
                   c_d &= \frac{2^{3+2/d}\left(1+\sqrt{d}\right)^2}{V_d^{2/d}},\\
                   V_d &= \frac{\pi^{d/2}}{\Gamma\left(\frac{d}{2}+1\right)},
                 \end{align}
                 and $\Gamma(\cdot)$ is the Gamma function defined for $x > 0$ by $\Gamma(x) = \int_{0}^{+\infty}t^{x-1}e^{-t}dt$.
\end{proposition}

\begin{proof}
  From Theorem 14.5 of \cite{bia:dev:15:nn} page 190, when Conditions (C\ref{condition:y:bounded}) to (C\ref{condition:variance:residuals}) hold, we have
       \begin{itemize}
       \item[] if $d=1$,
                $$\E_{\xi}\left[\left.
                 \left\{\widehat{m}^*_n(\xb) - m(\xb)\right\}^2 \right| Q_n  \right] = \sigma^2 \sum_{i \in U}w_i^2 + 8L^2 \sum_{i\in U} w_i \frac{i}{N},$$
        \item[] if $d \geq 2$,
                 $$\E_{\xi}\left[
                  \left\{\widehat{m}^*_n(\xb) - m(\xb)\right\}^2 \left| Q_n \right. \right] = \sigma^2 \sum_{i \in U}w_i^2 + c_d L^2 \sum_{i\in U} w_i \left(\frac{i}{N}\right)^{2/d}.$$
     \end{itemize}
    We need to condition on $Q_n$ here because otherwise, the weights are random because they are constructed based on $D_S$. The only quantities that are random in the right hand-sides are the $w_i$'s. For any unit $i \in U$ such that $ w_i\neq 0$, we have $i \leq \sum_{i \in U}\1{i  \in B\left(\xb, \rho_{k_nS}(\xb)\right)}$. Therefore, we have
    \begin{align}
      \sum_{i \in U}w_i^2 &= \left[\sum_{i \in U}  \1{i  \in B\left(\xb, \rho_{k_nS}(\xb)\right)}\right]^{-1} \leq \frac{1}{k_n},\\
      \sum_{i\in U} w_i \frac{i}{N} & \leq \frac{\sum_{i \in U}\1{i  \in B\left(\xb, \rho_{k_nS}(\xb)\right)}}{N},\\
      \sum_{i\in U} w_i \left(\frac{i}{N}\right)^{2/d} &\leq \left(\frac{\sum_{i \in U}\1{i  \in B\left(\xb, \rho_{k_nS}(\xb)\right)}}{N}\right)^{2/d}.\\
    \end{align}
    Applying condition (C\ref{condition:density:neighborhoods}) brings,
\begin{itemize}
       \item[] if $d=1$,
                $$\E_{\xi}\left[\left.\left\{\widehat{m}^*_n(\xb) - m(\xb)\right\}^2 \right| Q_n \right] \leq  \frac{\sigma^2}{k_n} + 8L^2 C \frac{k_n}{n},$$
        \item[] if $d \geq 2$,
                $$\E_{\xi}\left[\left. \left\{\widehat{m}^*_n(\xb) - m(\xb)\right\}^2 \right| Q_n \right] \leq \frac{\sigma^2}{k_n} + c_d L^2\left(C\frac{k_n}{n}\right)^{2/d} .$$
     \end{itemize}
     We conclude using that Condition (C\ref{condition:non:informative}) and the law of iterated expectations imply
     $$\E_{\xi}\E_{p}\left[\left\{\widehat{m}^*_n(\xb) - m(\xb)\right\}^2 \right]  = \E_{p}\E_{\xi}\left[\left\{\widehat{m}^*_n(\xb) - m(\xb)\right\}^2 \right] = \E_{p}\E_{\xi}\E_{\xi}\left[\left.\left\{\widehat{m}^*_n(\xb) - m(\xb)\right\}^2 \right|Q_n \right].$$
\end{proof}

\begin{proposition}\label{proposition:knn:design:consistency}
  Suppose that Conditions (C\ref{condition:y:bounded}) to (C\ref{condition:I2}) hold. We have
    \begin{align}
    \E_{p}\left[ \left\{\widehat{m}^*_n(\xb) - \widehat{m}_n(\xb)\right\}^2\right]
        &\leq 4 M_1^2\left[\frac{1}{k_n}  \left\{M_2 + M_3\E_{p}\left(\max\limits_{i,j\in U}|r_{ij}|\right)\right\} \right.\\
        &\qquad +
        \left.\frac{N}{k_n}\left\{\max\limits_{i,j\in U,i \neq j} \left|\frac{\pi_{ij}}{\pi_i \pi_j}-1 \right| + M_4\E_{p}\left(\max\limits_{i,j\in U}|r_{ij}|\right)  \right\}\right].
  \end{align}
  for some positive constants $M_1$ to $M_4$ with $\xi$-probability one. Moreover $ \E_{p}\left[ \left\{\widehat{m}^*_n(\xb) - \widehat{m}_n(\xb)\right\}^2\right]$ is $O(k_n^{-1})$ with $\xi$-probability one.
  The sample estimator $\widehat{m}_n(\xb)$ is therefore $L^2$-design-consistent for the hypothetical estimator $\widehat{m}^*_n(\xb)$.
\end{proposition}

Estimator $\widehat{m}_n(\xb)$ can be viewed as a local Horvitz-Thompson estimator computed with $k_n$ units. Hence, it is not surprising that its rate of convergence is of the order of $k_n^{-1}$. In order to prove Proposition~\ref{proposition:knn:design:consistency}, we will need the following Lemma.

\begin{lemma}\label{lemma:knn:design:consistency}
 Suppose that Condition~(C\ref{condition:pi}) holds. Then, there exists three constants $M_2>0$, $M_3>0$, and $M_4>0$ such that
 \begin{align}
   &  \max\limits_{i\in U} \left|\frac{1}{\pi_i}-1\right| \leq M_2, \\
   &  \max\limits_{i\in U} \left|\frac{1}{\pi_i}\left(\frac{1}{\pi_i}-2\right)\right| \leq M_3, \\
   &  \max\limits_{i,j\in U,i\neq j} \left| \frac{r_{ij}}{\pi_i\pi_j} - \frac{r_{ii}}{\pi_i}-\frac{r_{jj}}{\pi_j} \right| \leq M_4 \max\limits_{i,j\in U}|r_{ij}|,
 \end{align}
 for all $v$ and with $\xi$-probability one.
\end{lemma}

\begin{proof}
 Consider function $f(x) = \left|\frac{1}{x}-1\right|$ on domain $x \in [\lambda;1]$ where $\lambda$ is defined in Condition~(C\ref{condition:pi}). On this domain, we have $f(x) =\frac{1}{x}-1$. This function is decreasing for $x\in[\lambda;1]$. As a result, its maximum on this domain is attained at $x = \lambda$ and is $f(\lambda) = \frac{1}{\lambda} - 1 $. We set $M_2 = \frac{1}{\lambda} - 1$ and obtain the first inequality.

 Consider function $g(x) = \left|\frac{1}{x}\left(\frac{1}{x}-2\right)\right|$ on domain $x \in [\lambda;1]$ where $\lambda$ is defined in Condition~(C\ref{condition:pi}). Suppose for now that $\lambda \leq 0.5$. We have
 \begin{align}
    g(x)&=\left\{
   \begin{array}{ll}
     \frac{1}{x}\left(\frac{1}{x}-2\right), & \hbox{if $x\in[\lambda;0.5]$;} \\
     -\frac{1}{x}\left(\frac{1}{x}-2\right), & \hbox{if $x\in(0.5;1]$.}
   \end{array}
 \right.\end{align}
Function $g$ is decreasing for $x\in[\lambda;0.5]$. Therefore, its maximum on $[\lambda;0.5]$ is attained at $x = \lambda$ and is $g(\lambda) = \frac{1}{\lambda}\left(\frac{1}{\lambda}-2\right)$. Function $g$ is increasing for $x\in(0.5;1]$. Therefore, its maximum on this domain is attained at $x = 1$ and is $g(1)=1$. We set $M_3 = \max\left\{\frac{1}{\lambda}\left(\frac{1}{\lambda}-2\right),1 \right\}$ and obtain the second inequality. This equality is also verified when $\lambda > 0.5$. In this case, the maximum of $g$ on domain $x \in [\lambda;1]$ is attained at $x = 1$ and is $g(1)=1$.

Consider function $h(x,y)= \left( \frac{1}{x} + \frac{1}{y} + \frac{1}{xy} \right)$ defined on domain $(x,y) \in [\lambda;1]\times[\lambda;1]$. Function $h$ is positive and decreasing in both variables on this domain. Therefore, its maximum is attained at $(x,y)=(\lambda,\lambda)$ and is $h(\lambda,\lambda)= \frac{2}{\lambda}+\frac{1}{\lambda^2}$. We can write
\begin{align}
  \left| \frac{r_{ij}}{\pi_i\pi_j} - \frac{r_{ii}}{\pi_i}-\frac{r_{jj}}{\pi_j} \right| \leq \max\limits_{i,j\in U}|r_{ij}|h(\pi_i, \pi_j),
\end{align}
with $\pi_i, \pi_j$ both in domain $[\lambda;1]\times[\lambda;1]$ by Condition~(C\ref{condition:pi}). We set $M_4 = \frac{2}{\lambda}+\frac{1}{\lambda^2}$ and obtain the last inequality.
\end{proof}

\begin{proof}[Proof of Proposition \ref{proposition:knn:design:consistency}]
  We can write
  \begin{align}
    \widehat{m}^*_n(\xb) - \widehat{m}_n(\xb) = D_1 + D_2,
  \end{align}
  where
  \begin{align}
  D_1   &=  \left\{ \sum_{j \in U}  \1{j \in B\left(\xb, \rho_{k_nS}(\xb)\right)} \right\}^{-1} \sum_{i \in U}  \1{i \in B\left(\xb, \rho_{k_nS}(\xb)\right)}\left(1-\frac{I_i}{\pi_i}\right) Y_i ,\\
  D_2   &=  \left[\left\{ \sum_{j \in U}  \1{j \in B\left(\xb, \rho_{k_nS}(\xb)\right)} \right\}^{-1} - \left\{ \sum_{j \in U}  \1{j \in B\left(\xb, \rho_{k_nS}(\xb)\right)} \frac{I_j}{\pi_j} \right\}^{-1}\right] \sum_{i \in U}  \1{i \in B\left(\xb, \rho_{k_nS}(\xb)\right)}\frac{I_i}{\pi_i} Y_i\\
        &=\left\{ \sum_{j \in U}  \1{j \in B\left(\xb, \rho_{k_nS}(\xb)\right)} \frac{I_j}{\pi_j} -  \sum_{j \in U}  \1{j \in B\left(\xb, \rho_{k_nS}(\xb)\right)}  \right\}\\
            &\qquad \times \left\{ \sum_{j \in U}  \1{j \in B\left(\xb, \rho_{k_nS}(\xb)\right)} \right\}^{-1}\left\{ \sum_{j \in U}  \1{j \in B\left(\xb, \rho_{k_nS}(\xb)\right)} \frac{I_j}{\pi_j} \right\}^{-1} \sum_{i \in U}  \1{i \in B\left(\xb, \rho_{k_nS}(\xb)\right)}\frac{I_i}{\pi_i} Y_i\\
        &=\left\{ - \sum_{j \in U}  \1{j \in B\left(\xb, \rho_{k_nS}(\xb)\right)} \left(1-\frac{I_i}{\pi_i}\right) \right\}\\
            &\qquad \times \left\{ \sum_{j \in U}  \1{j \in B\left(\xb, \rho_{k_nS}(\xb)\right)} \right\}^{-1}\left\{ \sum_{j \in U}  \1{j \in B\left(\xb, \rho_{k_nS}(\xb)\right)} \frac{I_j}{\pi_j} \right\}^{-1} \sum_{i \in U}  \1{i \in B\left(\xb, \rho_{k_nS}(\xb)\right)}\frac{I_i}{\pi_i} Y_i.\\
  \end{align}
  We have
  \begin{align}\label{equation:sum:D1D2}
    \left\{\widehat{m}^*_n(\xb) - \widehat{m}_n(\xb)\right\}^2 &= \left(D_1 + D_2\right)^2 \leq 2\left( D_1^2 + D_2^2 \right).
  \end{align}
  The first term is
  \begin{align}
   D_1^2   &= \left\{ \sum_{j \in U}  \1{j \in B\left(\xb, \rho_{k_nS}(\xb)\right)} \right\}^{-2}\left\{ \sum_{i \in U}  \1{i \in B\left(\xb, \rho_{k_nS}(\xb)\right)}\left(1-\frac{I_i}{\pi_i}\right) Y_i   \right\}^2  \\
            &= \left\{ \sum_{j \in U}  \1{j \in B\left(\xb, \rho_{k_nS}(\xb)\right)} \right\}^{-2}\left\{ \sum_{i \in U}  \1{i \in B\left(\xb, \rho_{k_nS}(\xb)\right)}\left(1-\frac{I_i}{\pi_i}\right)^2 Y_i^2\right.\\
                    &\qquad\qquad    + \left. \sum_{i \in U}\sum_{j \in U: i \neq j}  \1{i,j \in B\left(\xb, \rho_{k_nS}(\xb)\right)}\left(1-\frac{I_i}{\pi_i}\right)\left(1-\frac{I_j}{\pi_j}\right) Y_iY_j  \right\}\\
            &= \left\{ \sum_{j \in U}  \1{j \in B\left(\xb, \rho_{k_nS}(\xb)\right)} \right\}^{-2}\left\{ \sum_{i \in U}  \1{i \in B\left(\xb, \rho_{k_nS}(\xb)\right)}\left(1 - 2\frac{I_i}{\pi_i} + \frac{I_i}{\pi_i^2}\right) Y_i^2\right.\\
                    &\qquad\qquad    + \left. \sum_{i \in U}\sum_{j \in U: i \neq j}  \1{i,j \in B\left(\xb, \rho_{k_nS}(\xb)\right)}\left(1-\frac{I_i}{\pi_i}-\frac{I_j}{\pi_j} + \frac{I_iI_j}{\pi_i\pi_j}\right) Y_iY_j  \right\}.
  \end{align}
  The second term is
\begin{align}
  D_2^2   &= \left\{ \sum_{j \in U}  \1{j \in B\left(\xb, \rho_{k_nS}(\xb)\right)} \left(1-\frac{I_i}{\pi_i}\right) \right\}^2
             \left\{ \sum_{j \in U}  \1{j \in B\left(\xb, \rho_{k_nS}(\xb)\right)} \right\}^{-2}\left\{ \sum_{j \in U}  \1{j \in B\left(\xb, \rho_{k_nS}(\xb)\right)} \frac{I_j}{\pi_j} \right\}^{-2} \\
             &\qquad\qquad \times\left\{ \sum_{i \in U}  \1{i \in B\left(\xb, \rho_{k_nS}(\xb)\right)}\frac{I_i}{\pi_i} Y_i  \right\}^2.
  \end{align}
  Condition~(C\ref{condition:y:bounded}) implies that
  \begin{align}
  D_2^2   &\leq M_1^2 \left\{ \sum_{j \in U}  \1{j \in B\left(\xb, \rho_{k_nS}(\xb)\right)} \left(1-\frac{I_i}{\pi_i}\right) \right\}^2
             \left\{ \sum_{j \in U}  \1{j \in B\left(\xb, \rho_{k_nS}(\xb)\right)} \right\}^{-2}\\
             &=  M_1^2 \left\{ \sum_{j \in U}  \1{j \in B\left(\xb, \rho_{k_nS}(\xb)\right)} \right\}^{-2}
             \left\{ \sum_{j \in U}  \1{j \in B\left(\xb, \rho_{k_nS}(\xb)\right)} \left(1-2\frac{I_i}{\pi_i}+\frac{I_i}{\pi_i^2}\right) \right.\\
             & \qquad\qquad \left.+ \sum_{i \in U}\sum_{j \in U: i \neq j}  \1{i,j \in B\left(\xb, \rho_{k_nS}(\xb)\right)}\left(1-\frac{I_i}{\pi_i}-\frac{I_j}{\pi_j} + \frac{I_iI_j}{\pi_i\pi_j}\right)\right\}.
  \end{align}
  Applying the law of iterated expectations, we obtain
  \begin{align}
    \E_{p}\left[ \left\{\widehat{m}^*_n(\xb) - \widehat{m}_n(\xb)\right\}^2\right]
        &\leq 2 \E_{p}\left( D_1^2 + D_2^2 \right) =  2 \E_{p}\left\{\E_{p} \left( D_1^2 + D_2^2 \right| \left.Q_n \right)\right\}\\
        &= 2 \E_{p}\left\{\E_{p} \left( D_1^2  \right| \left.Q_n \right)+ \E_{p} \left( D_2^2 \right| \left.Q_n \right)\right\}. \label{equation:law:iterated:expectations}
  \end{align}
  We will now introduce the random variables $r_{ij} = \E_p(I_iI_j|Q_n) - \pi_{ij}$ defined on page~\pageref{rij}. The inner expectation of the first term in Equation~\eqref{equation:law:iterated:expectations} can be rewritten as
  \begin{align}
    \E_{p} \left( D_1^2  \right| \left.Q_n \right) &= \left\{ \sum_{j \in U}  \1{j \in B\left(\xb, \rho_{k_nS}(\xb)\right)} \right\}^{-2}\left\{ \sum_{i \in U}  \1{i \in B\left(\xb, \rho_{k_nS}(\xb)\right)}\left(1 - 2\frac{r_{ii} + \pi_i}{\pi_i} + \frac{r_{ii} + \pi_i}{\pi_i^2}\right) Y_i^2\right.\\
                    &   + \left. \sum_{i \in U}\sum_{j \in U: i \neq j}  \1{i,j \in B\left(\xb, \rho_{k_nS}(\xb)\right)}\left(1-\frac{r_{ii} + \pi_i}{\pi_i}-\frac{r_{jj} + \pi_j}{\pi_j} + \frac{r_{ij} + \pi_{ij}}{\pi_i\pi_j}\right) Y_iY_j  \right\}\\
                    &= \left\{ \sum_{j \in U}  \1{j \in B\left(\xb, \rho_{k_nS}(\xb)\right)} \right\}^{-2}\left\{ \sum_{i \in U}  \1{i \in B\left(\xb, \rho_{k_nS}(\xb)\right)}\left(-1 - 2\frac{r_{ii}}{\pi_i} + \frac{r_{ii}}{\pi_i^2} +\frac{1}{\pi_i}\right) Y_i^2\right.\\
                    &   + \left. \sum_{i \in U}\sum_{j \in U: i \neq j}  \1{i,j \in B\left(\xb, \rho_{k_nS}(\xb)\right)}\left(-1-\frac{r_{ii}}{\pi_i}-\frac{r_{jj}}{\pi_j} + \frac{r_{ij}}{\pi_i\pi_j} + \frac{\pi_{ij}}{\pi_i\pi_j}\right) Y_iY_j  \right\}\\
                    &\leq \left\{ \sum_{j \in U}  \1{j \in B\left(\xb, \rho_{k_nS}(\xb)\right)} \right\}^{-2}\\
                    &\qquad \left\{ \sum_{i \in U}  \1{i \in B\left(\xb, \rho_{k_nS}(\xb)\right)}\left(\left|\frac{1}{\pi_i}-1\right| + \left| \frac{1}{\pi_i}\left(\frac{1}{\pi_i} - 2\right) \right|\left|r_{ii}\right| \right) Y_i^2\right.\\
                    &\qquad   + \left. \sum_{i \in U}\sum_{j \in U: i \neq j}  \1{i,j \in B\left(\xb, \rho_{k_nS}(\xb)\right)}\left(\left|\frac{\pi_{ij}}{\pi_i\pi_j} - 1\right| + \left|\frac{r_{ij}}{\pi_i\pi_j}-\frac{r_{ii}}{\pi_i}-\frac{r_{jj}}{\pi_j} \right|  \right) \left|Y_i\right|\left|Y_j\right|  \right\}.
  \end{align}
  Condition~(C\ref{condition:y:bounded}) and Lemma~\ref{lemma:knn:design:consistency} allow us to bound this quantity as follows
\begin{align}
  \E_{p} \left( D_1^2  \right| \left.Q_n \right) &\leq M_1^2\left\{ \sum_{j \in U}  \1{j \in B\left(\xb, \rho_{k_nS}(\xb)\right)} \right\}^{-2}\left\{ \left(M_2 + M_3\max\limits_{i,j\in U}|r_{ij}|\right)\sum_{i \in U}  \1{i \in B\left(\xb, \rho_{k_nS}(\xb)\right)} \right.\\
                    &   + \left. \left(\max\limits_{i,j\in U,i \neq j} \left|\frac{\pi_{ij}}{\pi_i \pi_j}-1 \right| + M_4\max\limits_{i,j\in U}|r_{ij}|  \right) \sum_{i \in U}\sum_{j \in U: i \neq j}  \1{i,j \in B\left(\xb, \rho_{k_nS}(\xb)\right)}   \right\}.
\end{align}
 Since $\sum_{j \in U}  \1{j \in B\left(\xb, \rho_{k_nS}(\xb)\right)} \geq k_n$ and $\sum_{i \in U}\sum_{j \in U, j\neq i} \1{i,j \in B\left(\xb, \rho_{k_nS}(\xb)\right)} \leq  N\sum_{i \in U}  \1{i \in B\left(\xb, \rho_{k_nS}(\xb)\right)}$, we can write
 \begin{align}
  \E_{p} \left( D_1^2  \right| \left.Q_n \right) &\leq M_1^2\left\{\frac{1}{k_n}  \left(M_2 + M_3\max\limits_{i,j\in U}|r_{ij}|\right) + \frac{N}{k_n}\left(\max\limits_{i,j\in U,i \neq j} \left|\frac{\pi_{ij}}{\pi_i \pi_j}-1 \right| + M_4\max\limits_{i,j\in U}|r_{ij}|  \right)\right\}.
\end{align}
Using similar computations, the inner expectation of the second term in Equation~\eqref{equation:law:iterated:expectations} can be bounded as follows
\begin{align}
  \E_{p} \left( D_2^2  \right| \left.Q_n \right) &\leq M_1^2\left\{\frac{1}{k_n}  \left(M_2 + M_3\max\limits_{i,j\in U}|r_{ij}|\right) + \frac{N}{k_n}\left(\max\limits_{i,j\in U,i \neq j} \left|\frac{\pi_{ij}}{\pi_i \pi_j}-1 \right| + M_4\max\limits_{i,j\in U}|r_{ij}|  \right)\right\}.
\end{align}
Inserting these two bounds into Equation~\eqref{equation:law:iterated:expectations} yields
  \begin{align}
    \E_{p}\left[ \left\{\widehat{m}^*_n(\xb) - \widehat{m}_n(\xb)\right\}^2\right]
        &\leq 4 M_1^2\left[\frac{1}{k_n}  \left\{M_2 + M_3\E_{p}\left(\max\limits_{i,j\in U}|r_{ij}|\right)\right\} \right.\\
        &\qquad +
        \left.\frac{N}{k_n}\left\{\max\limits_{i,j\in U,i \neq j} \left|\frac{\pi_{ij}}{\pi_i \pi_j}-1 \right| + M_4\E_{p}\left(\max\limits_{i,j\in U}|r_{ij}|\right)  \right\}\right].
  \end{align}
From Condition (C\ref{condition:I2}) the first term is $O(k_n^{-1})$ with $\xi$-probability one. From Conditions~(C\ref{condition:sampling:fraction}), (C\ref{condition:pi2}), and (C\ref{condition:I2}) the second term is also $O(k_n^{-1})$  with $\xi$-probability one. As a result \newline$\E_{p}\left[ \left\{\widehat{m}^*_n(\xb) - \widehat{m}_n(\xb)\right\}^2\right]$ is $O(k_n^{-1})$  with $\xi$-probability one.
\end{proof}

\end{document}